\documentclass[letterpaper]{article} 
\usepackage{aaai2027}  
\usepackage[hyphens]{url}  

\usepackage{natbib}  
\usepackage{caption} 
\usepackage{algorithm}
\usepackage{algorithmic}

\usepackage{booktabs, multirow, makecell, graphicx}

\usepackage{subcaption}    

\usepackage{amsfonts,amsmath,amsthm}

\newtheorem{proposition}{Proposition}
\newtheorem{theorem}{Theorem}
\newtheorem{lemma}{Lemma}
\usepackage{makecell}
\usepackage[table]{xcolor}
\usepackage{newfloat}
\usepackage{listings}
\DeclareCaptionStyle{ruled}{labelfont=normalfont,labelsep=colon,strut=off} 
\floatstyle{ruled}
\newfloat{listing}{tb}{lst}{}
\floatname{listing}{Listing}
\title{Learning to Difference: Adaptive Reversible Differencing (AdaRDiff) for Time Series Forecasting}

\author{
    Morad Laglil\textsuperscript{\rm 1,2},
    Younes Hlal\textsuperscript{\rm 1},
    Marouane El Hadari\textsuperscript{\rm 1},
    Emilie Devijver\textsuperscript{\rm 1},
    Eric Gaussier\textsuperscript{\rm 1}
}

\affiliations{
    \textsuperscript{\rm 1}Univ. Grenoble Alpes, CNRS, Grenoble INP, LIG, 38000 Grenoble, France\\
    \textsuperscript{\rm 2}Savoye, 21600 Longvic, France\\
    \{morad.laglil, emilie.devijver, eric.gaussier\}@univ-grenoble-alpes.fr,\\
    younes.hlal@etu.univ-grenoble-alpes.fr,\\
    marouane.el-hadari@grenoble-inp.org
}

\nocopyright

\begin{document}

\maketitle

\begin{abstract}
Reliable long-horizon time series forecasting is an important yet difficult problem. Trends and seasonality introduce complex temporal structure that challenges learning-based forecasting models. Differencing, which subtracts nearby past values to remove such structure, is the classical remedy, but its reliance on hand-picked orders and periods has kept it largely absent from recent deep architectures. We propose \textbf{\underline{Ada}}ptive \textbf{\underline{R}}eversible \textbf{\underline{Diff}}erencing \textbf{(AdaRDiff)}, a generalized differencing approach that uses learnable weights to simplify the series through weighted differencing with previous time instants. This yields stabilized residuals on which forecasting is performed, after which the removed components are restored autoregressively to reconstruct the forecast, capturing trend and seasonality jointly through a single operator. This reconstruction admits a closed-form convolutional expression, which parallelizes on GPU and yields up to $33.7\times$ speedup over the naive recurrence. 
We furthermore rely on a two-phase training schedule that separates temporal structure discovery from reconstruction learning, as suggested by a theoretical analysis of the gradient when using a linear forecasting model.
AdaRDiff attains state-of-the-art forecast accuracy across eight benchmarks spanning electricity, weather, traffic, and energy, at negligible parameter cost. Furthermore, it is designed as a plug-and-play module: integrating AdaRDiff improves eight diverse backbones, from linear models to Transformers, in the large majority of cases, by up to $25.9\%$ with a linear backbone and $18.3\%$ with iTransformer.
\end{abstract}

\section{Introduction}

Time series forecasting is fundamental to many real-world applications,
such as traffic flow prediction, product sales forecasting, and energy
consumption management, where accurate predictions support proactive
decision-making. Forecasting over long horizons, commonly referred to as
Long-term Time Series Forecasting (LTSF) \cite{informer}, is particularly
valuable and has accordingly drawn sustained attention. In all these
settings, forecasting quality depends on how well a model extracts and
exploits the temporal patterns present in the historical window
\cite{sparse}.

Among these patterns, trend and seasonality dominate: real-world series
combine long-term drifts and recurring periodic cycles into most of their
structured variation \cite{rb1990stl}, and these components are precisely
what makes forecasting difficult for learning-based models.
\citeauthor{NNFSTS}~\cite{NNFSTS} showed empirically that feedforward
networks struggle to capture trend and seasonal patterns from raw series,
whereas detrending and deseasonalizing the data beforehand substantially
reduces forecasting error. The classical tool for this is
\emph{differencing}: replacing each observation by its difference with a
previous value removes low-order trends (\textit{$1^{\text{st}}$-order
differencing}) or periodic components (\textit{seasonal differencing}),
leaving a simpler, more regular residual for the model to predict
\cite{OTexts}. Isolating such a residual is a powerful lever, and one that
architectural sophistication alone may not supply.

Such preprocessing is fundamental to classical statistical models, including ARIMA and SARIMA \cite{box1976time}, which have guided time series analysis for decades. Despite its proven effectiveness in traditional time series analysis, this principle has been largely overlooked in modern deep learning–based forecasting models. We believe this is due to a key limitation of classical differencing: the
scheme is rigid and must be specified by hand, including the orders, seasonal periods,
and their coefficients, which demands data exploration and domain expertise
and does not fit end-to-end training. AdaRDiff instead learns the scheme
directly from data, discovering the relevant lags and their weights rather
than requiring them to be fixed in advance.
Recent deep architectures have begun to reintroduce related
preprocessing through time series decomposition, as in DLinear \cite{dlinear} and CycleNet~\cite{lin2024cyclenet}, but these methods
target a fixed form of structure and do not adaptively capture the range of
trends and periodicities encountered across datasets.

This motivates an adaptive, data-driven form of differencing that captures
multiple seasonalities and trends jointly, and can be paired with any
forecasting model. We propose here such a framework, \textbf{Adaptive Reversible
Differencing (AdaRDiff)}. AdaRDiff is model-agnostic: rather than a
standalone forecaster, it is a differencing module that wraps around a
forecasting backbone, simplifying the series before the backbone predicts
and reconstructing the forecast afterward. It generalizes classical
differencing by replacing the hand-set coefficients with learnable,
channel-wise \emph{differencing weights}: for each time step, it forms a
weighted combination of past observations within a fixed window and
subtracts it from the current value, so that the model learns a
differencing scheme tailored to the data. The backbone then operates on the
resulting simplified residual. 
As AdaRDiff is \emph{reversible}, the forecast is then reconstructed
\emph{autoregressively}, applying the learned weights step by step to the
backbone's output so that each predicted value builds on the previously
reconstructed ones, recovering the forecast in the original space and
restoring the trend and seasonal structure that differencing had removed.
The learned differencing weights are moreover interpretable, as their
magnitudes at each lag reveal the dominant periods and 
dependencies the model has identified. 
AdaRDiff lets even a
lightweight Linear or MLP backbone reach state-of-the-art accuracy across
benchmarks spanning multiple domains.
Our main contributions are as follows:

\begin{itemize}
  \item We introduce AdaRDiff, a plug-and-play differencing module whose
  learnable weights capture trend and seasonality jointly and reversibly,
  recovering predictions in the original space while training end-to-end
  with any forecasting backbone.
\item We {theoretically} analyze the training dynamics of AdaRDiff with a linear backbone
  through the gradient with respect to its parameters, revealing a tension
  between the differencing and reconstruction pathways that motivates a
  two-phase training schedule separating temporal structure discovery from
  reconstruction learning.
  \item We evaluate AdaRDiff through extensive experiments on eight
  benchmark datasets and eight forecasting backbones, demonstrating its
  effectiveness.
\end{itemize}

\section{Related Work}
\label{sec:rel-work}

LTSF has a long statistical history, from exponential smoothing to the
Box--Jenkins ARIMA/SARIMA family \cite{box1976time}, built on
well-established statistical principles. We build on these but compare
against recent deep-learning approaches, grouped into lightweight linear/MLP
and Transformer-based models. Orthogonal to the choice of forecaster, reversible instance normalization
(RevIN) \cite{revin} has become a standard component of modern deep
forecasting architectures: it normalizes each input window and restores its
statistics at the output, mitigating distribution shift between training and
test data. AdaRDiff is fully compatible with RevIN and employs it as an
optional component.

\paragraph{Linear and MLP models.}
Lightweight linear- and MLP-based models often match or surpass far
larger architectures on LTSF at a fraction of their cost. DLinear
\cite{dlinear} decomposes the series into trend and seasonal parts by
moving average, and then forecasts each part with a linear model, while TimeMixer
\cite{wang2023timemixer} mixes multi-scale components through MLP blocks.
A second line exploits periodicity and low-rank structure to simplify the series and shrink the
model: CycleNet \cite{lin2024cyclenet} learns recurrent cycles and
subtracts them before forecasting the residual; SparseTSF
\cite{sparse} decouples periodicity from trend by downsampling
into cross-period subsequences, forecasting with under a thousand
parameters; TimeBase \cite{pmlr-v267-huang25az} extracts core basis
components and forecasts at the segment level rather than point by point;
and MixLinear \cite{ma2026mixlinear} combines segment-based trend
extraction in the time domain with low-rank spectral filtering in the
frequency domain, reaching competitive accuracy with roughly a hundred parameters. 
However those methods capture only a single fixed form of the input structure in isolation. In contrast, AdaRDiff, learns a single differencing operator that jointly captures trend and multiple seasonalities.  

\paragraph{Transformer-based models.}
The Transformer \cite{transformer}, effective at modeling long-range
dependencies, has been widely adapted to LTSF. Informer \cite{informer}
reduces attention complexity for long sequences, while FEDformer
\cite{fedformer} combines seasonal--trend decomposition with
frequency-domain attention, iTransformer \cite{liu2024itransformer}
treats each series as a token to model multivariate correlations, and
PatchTST \cite{patchtst} segments each series into subseries-level patches and applies attention
under channel independence. TQNet \cite{lin2025TQNet} uses periodically
shifted learnable queries in a single-layer attention module to capture
global inter-variable correlations. These models learn temporal or
cross-variate structure implicitly within attention. AdaRDiff instead
handles per-series temporal structure explicitly, in a model-agnostic
differencing module that can wrap a Transformer as readily as a linear
model.

A separate line of work develops time series foundation models
\cite{ansari2024chronos,10.5555/3692070.3692474}, pretrained on large corpora for
zero- and few-shot transfer. We do not compare against them: their
pretraining corpora often overlap with standard benchmarks, which precludes
a fair comparison with fully-supervised methods, a data-leakage issue
documented in prior work~\cite{laglil2026foundationmodelsfinetuningnew}.


\section{AdaRDiff: Adaptive Reversible Differencing}
\label{sec:method}

\paragraph{Problem setup.}
Given a look-back window of $L$ observations across $C$ channels, the task is
to forecast the next $H$ steps. Our method processes each channel
independently, so we present the univariate case ($C=1$) throughout: the
input is $\mathbf{x}=(x_1,\ldots,x_L)$, the ground-truth target window
$\mathbf{x}_{\mathrm{f}}=(x_{L+1},\ldots,x_{L+H})$, and the forecast
$\hat{\mathbf{x}}_{\mathrm{f}}=(\hat{x}_{L+1},\ldots,\hat{x}_{L+H})$.

\subsection{Adaptive Differencing and Reconstruction}
\label{sec:diff}
AdaRDiff is a model-agnostic, reversible framework for time series
forecasting.
It operates in three stages, summarized in Figure~\ref{fig:overview}:
\emph{(i)}~a learnable, channel-wise differencing operator
transforms the input into residuals;
\emph{(ii)}~a forecasting backbone maps past residuals to predicted
future residuals; and
\emph{(iii)}~an autoregressive reconstruction recovers the forecast in the original space.

\paragraph{Channel-Wise Differencing.}
Given $\mathbf{x}=(x_1,\ldots,x_L)$ and a window size~$P$, we learn
differencing weights $\boldsymbol{\delta} = (\delta_1,\ldots,\delta_P) \in \mathbb{R}^P$
and form the residuals $\mathbf{z}=(z_1,\ldots,z_{L-1})\in\mathbb{R}^{L-1}$ by
\begin{equation}
z_{t-1}
\;=\;
x_t - \sum_{j=1}^{P} \delta_j\, x_{t-j},
\qquad t \in \{2,\ldots,L\},
\label{eq:diff}
\end{equation}
In practice, the series is left-padded by repeating~$x_1$ a total of
$P{-}1$ times, i.e.\ $x_t:=x_1$ for $t\le1$, so that differencing applies
from $t=2$.
Since channels may exhibit heterogeneous dynamics, a distinct
$\boldsymbol{\delta}$ is learned per channel. This subsumes classical schemes: $\delta_1{=}1$ gives
first-order differencing, $\delta_1{=}2,\delta_2{=}{-}1$ second-order, and
$\delta_\ell{=}1$ seasonal differencing at lag~$\ell$ (all other coefficients
zero); combining lags lets a single operator remove an affine trend and a
periodic component at once.


\paragraph{Backbone.}
\label{sec:backbone}

The backbone $f_\theta$ maps past residuals to predicted future
residuals:
\begin{equation}
\hat{\mathbf{z}}
\;=\;
f_\theta(\mathbf{z})
\;\in\; \mathbb{R}^H.
\label{eq:backbone}
\end{equation}
We primarily use a linear model
($\hat{\mathbf{z}} = W_b\,\mathbf{z} + \mathbf{b}$,
$W_b \in \mathbb{R}^{H \times (L-1)}$)
and a two-layer MLP but more complex backbones are evaluated in
the ablation study (Section~\ref{sec:ablation}).

\paragraph{Autoregressive Reconstruction.} Given the predicted future residuals $\hat{\mathbf{z}}\in\mathbb{R}^H$,
the forecast in the original space is recovered by inverting
\eqref{eq:diff} through a linear recurrence: $\hat{x}_{L+1-j}:=x_{L+1-j}$ for
$j=1,\ldots,P$, and the forecast is recursively defined by
\begin{equation}
\hat{x}_{L+k}
=\hat{z}_k+\sum_{j=1}^{P}\delta_j\,\hat{x}_{L+k-j},
\qquad k=1,\ldots,H.
\label{eq:reverse}
\end{equation}

Computing $\hat{\mathbf{x}}_{\mathrm{f}}$ directly through the recursive application of Eq.~\eqref{eq:reverse} prevents parallelization across the horizon. We thus recast it as a convolution, amenable to parallelization. Let $\mathbf{c}=(c_1,\ldots,c_H)$ collect the
initial-condition terms, $c_k=\sum_{j=k}^{P}\delta_j\,x_{L+k-j}$ for
$k\le P$ and $c_k=0$ otherwise, and let $\mathbf{h}$ be the linear recurrence sequence \cite{elaydi2005difference} based on the same recurrence as the reconstruction, with $h_n=0$ for $n<0$, $h_0=1$, and
\begin{equation}
h_n = \sum_{j=1}^{P}\delta_j\,h_{n-j}, \qquad n\geq 1.
\label{eq:recurrence}
\end{equation}

\begin{proposition}
\label{prop:convolution}
The forecast~\eqref{eq:reverse} satisfies
\begin{equation}
\hat{x}_{L+k}=\bigl(\mathbf{h}*(\hat{\mathbf{z}}+\mathbf{c})\bigr)_k
=\sum_{m=0}^{k-1} h_m\,(\hat{z}_{k-m}+c_{k-m}).
\label{eq:convolution}
\end{equation}
\end{proposition}

In matrix notation, Eq.~\eqref{eq:convolution} can be written as  $\hat{\mathbf{x}}_{\mathrm{f}}=T_\mathbf{h}(\hat{\mathbf{z}}+\mathbf{c})$,
where $T_\mathbf{h}\in\mathbb{R}^{H\times H}$ is the lower-triangular
Toeplitz matrix with $(T_\mathbf{h})_{n,k}=h_{n-k}$. Both terms can be 
parallelized on GPUs: $\mathbf{h}$ is obtained through repeated squaring ~\cite{knuth1997taocp2} of the companion-matrix\footnote{The companion matrix
$\left(\begin{smallmatrix}
\delta_1 & \cdots & \delta_{P-1} & \delta_P\\
1 & & & 0\\
 & \ddots & & \vdots\\
 & & 1 & 0
\end{smallmatrix}\right)$.}
 of the recurrence in Eq.~\eqref{eq:recurrence}, requiring $O(\log H)$ sequential steps---against $O(H)$ of the naive recurrence---and the Toeplitz product is realized as a  one-dimensional convolution across horizon and batch, as detailed in Appendix B.
 
 \begin{figure}[t]
  \centering
  \includegraphics[width=1\linewidth]{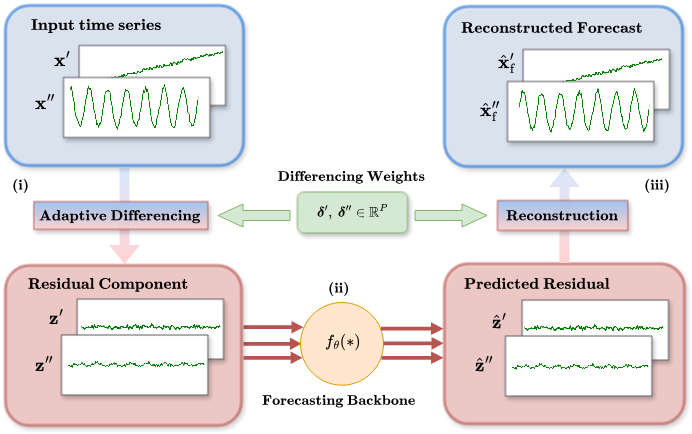}
  \caption{%
    \textbf{Overview of AdaRDiff.}
    \textbf{(i)}~A learnable differencing operator transforms the
    input into residuals.
    \textbf{(ii)}~A forecasting backbone predicts future residuals.
    \textbf{(iii)}~The forecast is reconstructed autoregressively in
    the original space.}
  \label{fig:overview}
\end{figure}
\paragraph{Training.}
The full model, with parameters $(\boldsymbol{\delta},\theta)$, is trained
end-to-end, the target objective being the forecast loss
\begin{equation}
\mathcal{L}_{\mathrm{forecast}}
=\tfrac12\|\hat{\mathbf{x}}_{\mathrm{f}}-\mathbf{x}_{\mathrm{f}}\|^2 ,
\label{eq:forecast_loss}
\end{equation}
Since the reconstruction~\eqref{eq:convolution}
depends polynomially on $\boldsymbol{\delta}$, this objective is non-convex
and we therefore optimize it by stochastic
gradient descent. 

\subsection{Gradient Analysis with Linear Backbone}
\label{sec:gradient_analysis}

We study the training dynamics of AdaRDiff through the gradient with
respect to its parameters, taking the linear backbone as the
analytically tractable case. Throughout, $\mathcal{L}:=\mathcal{L}_{\mathrm{forecast}}$.

\paragraph{Backbone gradient.}
With $e:=\hat{\mathbf{x}}_{\mathrm{f}}-\mathbf{x}_{\mathrm{f}}$, the
linear-backbone gradient is
\begin{equation}
\nabla_{W_b}\mathcal{L}=(T_\mathbf{h}^\top e)\,\mathbf{z}^\top,
\qquad
(T_\mathbf{h}^\top e)_t=\sum_{k=t}^{H} h_{k-t}\,e_k,
\label{eq:grad_Wb}
\end{equation}
so the effective error at position $t$ is a backward convolution of $e$
with $\mathbf{h}$. Without reconstruction ($T_\mathbf{h}=I$) this
collapses to $e\,\mathbf{z}^\top$, where each position sees only its own
error; reconstruction couples positions through $\mathbf{h}$, whose
growth is governed by the spectral radius $\rho$ of the companion matrix.

\begin{proposition}[Growth of $h_n$ and Error Propagation]
\label{prop:stability}
For $(h_n)$ as in Eq.~\eqref{eq:recurrence}\footnote{The case $\rho=1$ occurs on a measure-zero set of parameters and
is omitted here for readability; it is treated in
Appendix~A.\label{fn:rho=1}}:
\begin{enumerate}
 \item If $\rho<1$: $\forall\rho'\in(\rho,1),\ \exists C>0,\ 
|h_n|\le C\rho'^{\,n}$, and $|(T_\mathbf{h}^\top e)_t|=O(\|e\|)$; 
\item If $\rho>1$, $\sup_n|h_n|=\infty$. 
\end{enumerate}
\end{proposition}

The coupling is thus localized when $\rho<1$: each position receives
feedback dominated by nearby errors, enabling position-specific learning, and the effective error is bounded independently of the horizon. This uniform control is lost as $\rho$ exceeds $1$, in which case $\mathbf{h}$ is no longer bounded.

\paragraph{{Differencing }Gradient.}
The gradient of $\mathcal{L}$ in $\delta_j$ decomposes into three terms,
one per pathway through which $\boldsymbol{\delta}$ acts on the forecast:
\begin{equation}
\frac{\partial\mathcal{L}}{\partial\delta_j}
=
\underbrace{\sum_{n=0}^{H-1} \frac{\partial h_n}{\partial\delta_j}\,\phi_{ev}(n)}_{\text{I: reconstruction}}
-\underbrace{\mathbf{p}^\top \mathbf{x}^{(j)}_{\scriptscriptstyle 2:L}}_{\text{II: differencing}}
+\underbrace{\sum_{t=1}^{j}(T_\mathbf{h}^\top e)_t\,x_{L+t-j}}_{\text{III: boundary}},
\label{eq:grad_delta}
\end{equation}
where $\phi_{ev}(n)=\sum_{k=n+1}^{H} e_k\,v_{k-n}$ with
$v_m=\hat{z}_m+c_m$, $\mathbf{p}^\top=(T_\mathbf{h}^\top e)^\top W_b$, and
$\mathbf{x}^{(j)}_{\scriptscriptstyle 2:L}\in\mathbb{R}^{L-1}$ is the lag-$j$
shift of $\mathbf{x}_{\scriptscriptstyle 2:L}$, i.e.\
$(x^{(j)}_{\scriptscriptstyle 2:L})_s=x_{s+1-j}$.\footnote{We adopt the
left-padding convention of Eq.~\eqref{eq:diff} ($x_t:=x_1$ for $t\le 1$), so
that all shifted indices are well-defined.}

As the third term of Eq.~\eqref{eq:grad_delta} arises only because the first $P$ forecast steps are reconstructed
from the last observed context values. It is supported on at most $P$ of the
$H$ horizon positions, a vanishing fraction for long horizons; we therefore focus on Terms~I and~II of Eq.~\eqref{eq:grad_delta}.

\textbf{Term~I} is the only term involving the sensitivity of $\mathbf{h}$
to $\boldsymbol{\delta}$, whose behavior is again governed by $\rho$.

\begin{proposition}[Growth of $\partial h_n/\partial\delta_j$ and Term~I]
\label{prop:sensitivity}
For $j\in\{1,\ldots,P\}$, the sensitivity
$g_n^{(j)}:=\partial h_n/\partial\delta_j$ is the self-convolution
$g_n^{(j)}=(\mathbf{h}*\mathbf{h}^{(j)})_n$ with $h_n^{(j)}:=h_{n-j}$,
and satisfies:\footref{fn:rho=1}
\begin{enumerate}
\item If $\rho<1$: $\forall\rho'\in(\rho,1),\ \exists C_g>0,\ \forall n\ge0,\
|g_n^{(j)}|\le C_g\,\rho'^{\,n}$, and
$|\mathrm{Term~I}|=O(\|e\|\,\|\mathbf{v}\|)$; 
\item If $\rho>1$, $\sup_n|g_n^{(j)}|=\infty$. 
\end{enumerate}
\end{proposition}
Term~I is therefore bounded independently of the horizon only in the stable regime ($\rho < 1)$.
Since nothing in the objective prevents $\boldsymbol{\delta}$ from leaving
the stable regime, we introduce an optional reparameterization that makes
$\rho$ explicitly controllable, as stated in Theorem~\ref{thm:l1}.
\begin{theorem}
\label{thm:l1}
If $\|\boldsymbol{\delta}\|_1<1$ in Eq.~\eqref{eq:recurrence}, then $\rho<1$.
\end{theorem}
Concretely, we $\ell_1$-normalize an unconstrained
$\tilde{\boldsymbol{\delta}}\in\mathbb{R}^P$ and rescale by a gain $\alpha$,
\begin{equation}
\delta_j^{\mathrm{eff}}
=\frac{\alpha\,\tilde{\delta}_j}{\|\tilde{\boldsymbol{\delta}}\|_1},
\label{eq:reparam}
\end{equation}
so that $\|\boldsymbol{\delta}^{\mathrm{eff}}\|_1=|\alpha|$: direction and
magnitude are decoupled, and $\alpha$ directly controls proximity to the
stability condition. Fixing $\alpha<1$ would enforce $\rho<1$ but imposes
the same magnitude on every channel, whereas channels may require different
amounts of differencing; we therefore learn $\alpha$ per channel. Note that we consider the reparametrization as an optional procedure, which may not be needed on all datasets and forecasting configurations.

\textbf{Term~II}, from the differencing pathway, does not involve the sensitivity $g_n^{(j)}$. As stated in Proposition~\ref{prop:autocov}, its expected value is a linear functional of the input autocovariance in which the lag $j$ enters only through data autocovariance, so the learning signal for
$\boldsymbol{\delta}$ is also shaped by the data's second-order temporal
structure---trend and periodicity---a property specific to this pathway.

\begin{proposition}[Autocovariance coupling]
\label{prop:autocov}
Let $(x_t)$ be zero-mean weakly stationary with autocovariance
$\gamma(k)=\mathbb{E}[x_t\,x_{t-k}]$. Then, for any fixed parameters
$(\boldsymbol{\delta},W_b,\mathbf{b})$, the differencing term of
Eq.~\eqref{eq:grad_delta} has expectation
\begin{equation}
\mathbb{E}_{\mathbf{x}}\bigl[\mathbf{p}^{\top}
\mathbf{x}^{(j)}_{\scriptscriptstyle 2:L}\bigr]
=\sum_{m=1-L}^{L+H-2}\Phi(m)\,\gamma(m+j),
\label{eq:autocov}
\end{equation}
where the weights $\Phi(m)$ depend on the parameters but not on the lag~$j$.
\end{proposition}

There is a fundamental tension between Terms I and II: Term II encourages $\boldsymbol{\delta}$ to match the data autocorrelation structure, potentially implying large \(\|\boldsymbol{\delta}\|_1\), whereas Term I enforces stable reconstruction dynamics, requiring smaller \(\|\boldsymbol{\delta}\|_1\). The learned weights \(\boldsymbol{\delta}\) are therefore a compromise which depends on the relative magnitudes of these two contributions.


\subsection{Two-Phase Training}
\label{sec:training}

We address the tension between Terms I and II with a two-phase training scheme that separates temporal structure discovery from reconstruction learning.

\paragraph{Phase~1: residual loss.}
We first train the differencing operator and the backbone in residual space,
bypassing reconstruction entirely: the backbone output $\hat{\mathbf{z}}$ is
compared directly against the differenced target, minimizing
\begin{equation}
\mathcal{L}_{\mathrm{res}}
=\tfrac12\|\hat{\mathbf{z}}-\mathbf{z}_{\mathrm{f}}\|^2,
\qquad
z_{\mathrm{f},k}=x_{L+k}-\sum_{j=1}^{P}\delta_j\,x_{L+k-j},
\label{eq:proxy_loss}
\end{equation}
where $\mathbf{z}_{\mathrm{f}}$ applies the differencing operator to the
target window. Because $T_\mathbf{h}$ is absent from the forward pass,
Term~I vanishes and, with
$e_{\mathrm{p}}:=\hat{\mathbf{z}}-\mathbf{z}_{\mathrm{f}}$, the gradient
reduces to 

\begin{equation}
\frac{\partial\mathcal{L}_{\mathrm{res}}}{\partial\delta_j}
=e_{\mathrm{p}}^\top\bigl(\mathbf{x}_{\mathrm{f}}^{(j)}
-W_b\,\mathbf{x}^{(j)}_{\scriptscriptstyle 2:L}\bigr)
\label{eq:reduced_loss}
\end{equation}
with $\mathbf{x}_{\mathrm{f}}^{(j)}\in\mathbb{R}^{H}$ the lag-$j$ shift of the
target window, $(x_{\mathrm{f}}^{(j)})_k=x_{L+k-j}$. This gradient carries no
sensitivity $g_n^{(j)}$ and is bounded by
$O\!\bigl((\|W_b\|\,\|\mathbf{x}\|+\|\mathbf{x}_{\mathrm{f}}\|)
\|e_{\mathrm{p}}\|\bigr)$, independently of $\rho$, so Phase~1 fits the data
structure free of the reconstruction's conditioning.

\definecolor{OursColor}{HTML}{8B0000}
\begin{table*}[!ht]
    \centering
    \fontsize{10pt}{10pt}\selectfont
    \scalebox{0.5}{
    \begin{tabular}{l | c |c c|c c|c c|c c|c c|c c|c c|c c|c c|c c|c c|c c|c c}
    \toprule[2pt]
       \multicolumn{2}{c|}{Models} & \multicolumn{2}{c|}{\makecell{\textcolor{OursColor}{\textbf{AdaRDiff-Linear}}\\(Ours)}} & \multicolumn{2}{c|}{\makecell{\textcolor{OursColor}{\textbf{AdaRDiff-MLP}}\\(Ours)}} & \multicolumn{2}{c|}{\makecell{MixLinear\\(2026)}} & \multicolumn{2}{c|}{\makecell{TimeBase\\(2025)}} & \multicolumn{2}{c|}{\makecell{TQNet\\(2025)}} & \multicolumn{2}{c|}{\makecell{SparseTSF\\(2024)}} & \multicolumn{2}{c|}{\makecell{TimeMixer\\(2024)}} & \multicolumn{2}{c|}{\makecell{iTransformer\\(2024)}} & \multicolumn{2}{c|}{\makecell{CycleNet\\(2024)}} & \multicolumn{2}{c|}{\makecell{DLinear\\(2023)}} & \multicolumn{2}{c|}{\makecell{PatchTST\\(2023)}} & \multicolumn{2}{c|}{\makecell{TimesNet\\(2023)}} & \multicolumn{2}{c}{\makecell{FEDformer\\(2022)}} \\
       \cmidrule(lr){1-2}\cmidrule(lr){3-28}
       \multicolumn{2}{c|}{Metrics} & MSE & MAE & MSE & MAE & MSE & MAE & MSE & MAE & MSE & MAE & MSE & MAE & MSE & MAE & MSE & MAE & MSE & MAE & MSE & MAE & MSE & MAE & MSE & MAE & MSE & MAE \\
       \midrule
         & 96   & 0.360 & 0.390 & 0.389 & 0.412 & \underline{\textcolor{blue}{0.351}} & \textcolor{red}{\textbf{0.383}} & \textcolor{red}{\textbf{0.349}} & \underline{\textcolor{blue}{0.384}} & 0.377 & 0.404 & 0.362 & 0.389 & 0.410 & 0.441 & 0.389 & 0.421 & 0.379 & 0.403 & 0.378 & 0.402 & 0.377 & 0.408 & 0.437 & 0.454 & 0.485 & 0.500   \\
         & 192   & 0.402 & 0.419 & 0.409 & 0.434 & \underline{\textcolor{blue}{0.395}} & 0.423 & \textcolor{red}{\textbf{0.387}} & \textcolor{red}{\textbf{0.410}} & 0.424 & 0.437 & 0.404 & \underline{\textcolor{blue}{0.412}} & 0.448 & 0.465 & 0.424 & 0.446 & 0.416 & 0.425 & 0.415 & 0.425 & 0.413 & 0.431 & 0.456 & 0.469 & 0.481 & 0.498   \\
         & 336   & 0.426 & 0.436 & 0.449 & 0.446 & \underline{\textcolor{blue}{0.411}} & 0.434 & \textcolor{red}{\textbf{0.408}} & \textcolor{red}{\textbf{0.418}} & 0.450 & 0.454 & 0.435 & \underline{\textcolor{blue}{0.428}} & 0.482 & 0.490 & 0.456 & 0.469 & 0.447 & 0.445 & 0.449 & 0.449 & 0.436 & 0.446 & 0.494 & 0.494 & 0.522 & 0.521   \\
       \multirow{-4}*{\rotatebox{90}{ETTh1}}  & 720   & \textcolor{red}{\textbf{0.423}} & 0.450 & 0.487 & 0.485 & \textcolor{red}{\textbf{0.423}} & 0.456 & 0.439 & \textcolor{red}{\textbf{0.446}} & 0.521 & 0.517 & \underline{\textcolor{blue}{0.426}} & \underline{\textcolor{blue}{0.448}} & 0.475 & 0.500 & 0.545 & 0.532 & 0.477 & 0.483 & 0.507 & 0.517 & 0.455 & 0.475 & 0.632 & 0.578 & 0.604 & 0.575   \\
         \midrule
         & 96   & \textcolor{red}{\textbf{0.268}} & \textcolor{red}{\textbf{0.333}} & 0.272 & 0.340 & 0.283 & 0.340 & 0.292 & 0.345 & 0.282 & 0.350 & 0.294 & 0.346 & 0.315 & 0.380 & 0.305 & 0.361 & \underline{\textcolor{blue}{0.271}} & \underline{\textcolor{blue}{0.337}} & 0.294 & 0.360 & 0.276 & 0.339 & 0.349 & 0.403 & 0.401 & 0.451   \\
         & 192   & \textcolor{red}{\textbf{0.330}} & \textcolor{red}{\textbf{0.374}} & \textcolor{red}{\textbf{0.330}} & 0.378 & 0.337 & \underline{\textcolor{blue}{0.376}} & 0.339 & 0.387 & 0.354 & 0.393 & 0.340 & 0.377 & 0.383 & 0.415 & 0.405 & 0.421 & \underline{\textcolor{blue}{0.332}} & 0.380 & 0.412 & 0.437 & 0.342 & 0.385 & 0.500 & 0.488 & 0.425 & 0.464   \\
         & 336   & \textcolor{red}{\textbf{0.352}} & 0.399 & \textcolor{red}{\textbf{0.352}} & 0.404 & \underline{\textcolor{blue}{0.356}} & \textcolor{red}{\textbf{0.395}} & 0.358 & 0.410 & 0.387 & 0.419 & 0.360 & \underline{\textcolor{blue}{0.398}} & 0.385 & 0.438 & 0.411 & 0.436 & 0.362 & 0.408 & 0.471 & 0.478 & 0.364 & 0.405 & 0.445 & 0.465 & 0.427 & 0.471   \\
       \multirow{-4}*{\rotatebox{90}{ETTh2}}  & 720   & \textcolor{red}{\textbf{0.378}} & \underline{\textcolor{blue}{0.424}} & 0.398 & 0.445 & \underline{\textcolor{blue}{0.380}} & \textcolor{red}{\textbf{0.423}} & 0.400 & 0.448 & 0.415 & 0.447 & 0.383 & 0.425 & 0.432 & 0.471 & 0.448 & 0.470 & 0.415 & 0.449 & 0.740 & 0.609 & 0.395 & 0.434 & 0.438 & 0.465 & 0.462 & 0.493   \\
         \midrule
         & 96   & \textcolor{red}{\textbf{0.286}} & \textcolor{red}{\textbf{0.336}} & \underline{\textcolor{blue}{0.291}} & \underline{\textcolor{blue}{0.345}} & 0.333 & 0.375 & 0.311 & 0.351 & 0.298 & 0.352 & 0.314 & 0.359 & 0.332 & 0.384 & 0.315 & 0.369 & 0.307 & 0.353 & 0.314 & 0.350 & 0.298 & 0.352 & 0.359 & 0.391 & 0.406 & 0.441   \\
         & 192   & \textcolor{red}{\textbf{0.326}} & \textcolor{red}{\textbf{0.360}} & \underline{\textcolor{blue}{0.327}} & \underline{\textcolor{blue}{0.363}} & 0.353 & 0.382 & 0.338 & 0.371 & 0.347 & 0.381 & 0.348 & 0.376 & 0.355 & 0.398 & 0.349 & 0.388 & 0.337 & 0.371 & 0.347 & 0.381 & 0.335 & 0.373 & 0.368 & 0.398 & 0.450 & 0.477   \\
         & 336   & \textcolor{red}{\textbf{0.356}} & \textcolor{red}{\textbf{0.378}} & \underline{\textcolor{blue}{0.362}} & \underline{\textcolor{blue}{0.384}} & 0.391 & 0.407 & 0.364 & 0.386 & 0.373 & 0.395 & 0.368 & 0.386 & 0.386 & 0.416 & 0.381 & 0.409 & 0.364 & 0.387 & 0.367 & 0.387 & 0.366 & 0.394 & 0.429 & 0.438 & 0.436 & 0.466   \\
       \multirow{-4}*{\rotatebox{90}{ETTm1}}  & 720   & \textcolor{red}{\textbf{0.410}} & \textcolor{red}{\textbf{0.408}} & 0.421 & 0.416 & 0.439 & 0.435 & \underline{\textcolor{blue}{0.413}} & 0.414 & 0.437 & 0.429 & 0.419 & 0.413 & 0.452 & 0.457 & 0.437 & 0.439 & \textcolor{red}{\textbf{0.410}} & \underline{\textcolor{blue}{0.411}} & 0.415 & 0.415 & 0.420 & 0.421 & 0.477 & 0.474 & 0.462 & 0.479   \\
         \midrule
         & 96   & \textcolor{red}{\textbf{0.159}} & \underline{\textcolor{blue}{0.250}} & \underline{\textcolor{blue}{0.160}} & \underline{\textcolor{blue}{0.250}} & 0.170 & 0.261 & 0.162 & 0.256 & 0.167 & 0.257 & 0.167 & 0.259 & 0.192 & 0.285 & 0.179 & 0.274 & \textcolor{red}{\textbf{0.159}} & \textcolor{red}{\textbf{0.249}} & 0.163 & 0.257 & 0.165 & 0.260 & 0.200 & 0.288 & 0.339 & 0.406   \\
         & 192   & \underline{\textcolor{blue}{0.216}} & \textcolor{red}{\textbf{0.289}} & \textcolor{red}{\textbf{0.214}} & \underline{\textcolor{blue}{0.291}} & 0.223 & 0.298 & 0.218 & 0.293 & 0.220 & 0.295 & 0.219 & 0.297 & 0.253 & 0.329 & 0.239 & 0.314 & \textcolor{red}{\textbf{0.214}} & \textcolor{red}{\textbf{0.289}} & 0.223 & 0.304 & 0.219 & 0.298 & 0.274 & 0.337 & 0.397 & 0.452   \\
         & 336   & \underline{\textcolor{blue}{0.268}} & \underline{\textcolor{blue}{0.327}} & \textcolor{red}{\textbf{0.265}} & \underline{\textcolor{blue}{0.327}} & 0.275 & 0.332 & 0.270 & 0.328 & 0.283 & 0.340 & 0.271 & 0.330 & 0.307 & 0.362 & 0.309 & 0.356 & \underline{\textcolor{blue}{0.268}} & \textcolor{red}{\textbf{0.326}} & 0.291 & 0.355 & \underline{\textcolor{blue}{0.268}} & 0.333 & 0.340 & 0.382 & 0.449 & 0.491   \\
       \multirow{-4}*{\rotatebox{90}{ETTm2}}  & 720   & \underline{\textcolor{blue}{0.350}} & \textcolor{red}{\textbf{0.380}} & \textcolor{red}{\textbf{0.345}} & \underline{\textcolor{blue}{0.381}} & 0.360 & 0.384 & 0.352 & \textcolor{red}{\textbf{0.380}} & 0.370 & 0.391 & 0.353 & \textcolor{red}{\textbf{0.380}} & 0.380 & 0.412 & 0.387 & 0.407 & 0.353 & 0.384 & 0.407 & 0.433 & 0.352 & 0.386 & 0.384 & 0.407 & 0.451 & 0.499   \\
         \midrule
         & 96   & \textcolor{red}{\textbf{0.140}} & \textcolor{red}{\textbf{0.192}} & \underline{\textcolor{blue}{0.141}} & 0.201 & 0.172 & 0.227 & 0.146 & \underline{\textcolor{blue}{0.198}} & 0.162 & 0.216 & 0.174 & 0.231 & 0.163 & 0.223 & 0.159 & 0.212 & 0.164 & 0.220 & 0.174 & 0.242 & 0.149 & 0.199 & 0.176 & 0.234 & 0.289 & 0.342   \\
         & 192   & \textcolor{red}{\textbf{0.181}} & \textcolor{red}{\textbf{0.234}} & \underline{\textcolor{blue}{0.183}} & 0.245 & 0.212 & 0.260 & 0.185 & \underline{\textcolor{blue}{0.241}} & 0.200 & 0.249 & 0.216 & 0.267 & 0.201 & 0.254 & 0.203 & 0.252 & 0.209 & 0.258 & 0.215 & 0.277 & 0.193 & 0.243 & 0.219 & 0.270 & 0.340 & 0.394   \\
         & 336   & \textcolor{red}{\textbf{0.234}} & \textcolor{red}{\textbf{0.276}} & \textcolor{red}{\textbf{0.234}} & 0.288 & 0.257 & 0.295 & \underline{\textcolor{blue}{0.236}} & \underline{\textcolor{blue}{0.281}} & 0.251 & 0.288 & 0.260 & 0.299 & 0.258 & 0.300 & 0.253 & 0.291 & 0.242 & 0.283 & 0.262 & 0.319 & 0.240 & \underline{\textcolor{blue}{0.281}} & 0.277 & 0.311 & 0.370 & 0.408   \\
       \multirow{-4}*{\rotatebox{90}{Weather}}  & 720   & \textcolor{red}{\textbf{0.302}} & 0.342 & \underline{\textcolor{blue}{0.303}} & 0.339 & 0.324 & 0.344 & 0.309 & \textcolor{red}{\textbf{0.331}} & 0.318 & 0.339 & 0.325 & 0.345 & 0.329 & 0.348 & 0.317 & 0.337 & 0.338 & \underline{\textcolor{blue}{0.333}} & 0.319 & 0.359 & 0.312 & 0.334 & 0.344 & 0.356 & 0.420 & 0.421   \\
         \midrule
         & 96   & 0.131 & \underline{\textcolor{blue}{0.225}} & \textcolor{red}{\textbf{0.128}} & \underline{\textcolor{blue}{0.225}} & 0.139 & 0.233 & 0.139 & 0.231 & 0.138 & 0.237 & 0.139 & 0.239 & 0.142 & 0.247 & 0.135 & 0.233 & \textcolor{red}{\textbf{0.128}} & \textcolor{red}{\textbf{0.223}} & 0.141 & 0.244 & 0.141 & 0.240 & 0.202 & 0.308 & 0.226 & 0.341   \\
         & 192   & \underline{\textcolor{blue}{0.146}} & \underline{\textcolor{blue}{0.241}} & \underline{\textcolor{blue}{0.146}} & 0.244 & 0.154 & 0.248 & 0.153 & 0.245 & 0.150 & 0.246 & 0.155 & 0.250 & 0.164 & 0.275 & 0.155 & 0.253 & \textcolor{red}{\textbf{0.143}} & \textcolor{red}{\textbf{0.237}} & 0.155 & 0.258 & 0.156 & 0.256 & 0.218 & 0.322 & 0.220 & 0.336   \\
         & 336   & \underline{\textcolor{blue}{0.161}} & \underline{\textcolor{blue}{0.259}} & \underline{\textcolor{blue}{0.161}} & 0.261 & 0.171 & 0.265 & 0.169 & 0.262 & 0.184 & 0.277 & 0.171 & 0.265 & 0.171 & 0.260 & 0.169 & 0.267 & \textcolor{red}{\textbf{0.159}} & \textcolor{red}{\textbf{0.254}} & 0.170 & 0.275 & 0.172 & 0.267 & 0.232 & 0.332 & 0.224 & 0.337   \\
       \multirow{-4}*{\rotatebox{90}{Electricity}}  & 720   & \textcolor{red}{\textbf{0.196}} & 0.295 & \textcolor{red}{\textbf{0.196}} & 0.296 & 0.209 & 0.298 & 0.207 & \underline{\textcolor{blue}{0.294}} & 0.207 & 0.296 & 0.208 & 0.300 & 0.209 & 0.313 & 0.204 & 0.301 & \underline{\textcolor{blue}{0.197}} & \textcolor{red}{\textbf{0.287}} & 0.209 & 0.309 & 0.208 & 0.299 & 0.299 & 0.375 & 0.271 & 0.378   \\
         \midrule
         & 96   & 0.387 & 0.270 & \underline{\textcolor{blue}{0.367}} & \underline{\textcolor{blue}{0.265}} & 0.390 & 0.268 & 0.394 & 0.267 & 0.381 & 0.268 & 0.389 & 0.268 & 0.396 & 0.294 & 0.374 & 0.273 & 0.381 & 0.266 & 0.396 & 0.272 & \textcolor{red}{\textbf{0.363}} & \textcolor{red}{\textbf{0.250}} & 0.605 & 0.325 & 0.664 & 0.431   \\
         & 192   & 0.397 & 0.273 & \textcolor{red}{\textbf{0.382}} & 0.272 & 0.404 & 0.276 & 0.403 & \underline{\textcolor{blue}{0.271}} & 0.396 & 0.277 & 0.399 & 0.272 & 0.404 & 0.295 & \underline{\textcolor{blue}{0.393}} & 0.283 & 0.394 & 0.273 & 0.404 & 0.275 & \textcolor{red}{\textbf{0.382}} & \textcolor{red}{\textbf{0.258}} & 0.627 & 0.340 & 0.613 & 0.382   \\
         & 336   & 0.412 & 0.280 & \textcolor{red}{\textbf{0.396}} & 0.279 & 0.416 & 0.283 & 0.417 & 0.278 & \underline{\textcolor{blue}{0.397}} & \underline{\textcolor{blue}{0.277}} & 0.417 & 0.279 & 0.419 & 0.302 & 0.409 & 0.292 & 0.406 & 0.279 & 0.417 & 0.283 & 0.399 & \textcolor{red}{\textbf{0.268}} & 0.631 & 0.349 & 0.612 & 0.379   \\
       \multirow{-4}*{\rotatebox{90}{Traffic}}  & 720   & 0.449 & 0.298 & \underline{\textcolor{blue}{0.434}} & 0.298 & 0.452 & \textcolor{red}{\textbf{0.283}} & 0.456 & 0.298 & 0.439 & 0.299 & 0.449 & 0.299 & 0.458 & 0.309 & 0.450 & 0.314 & 0.441 & 0.300 & 0.457 & 0.310 & \textcolor{red}{\textbf{0.432}} & \underline{\textcolor{blue}{0.289}} & 0.700 & 0.371 & 0.664 & 0.410   \\
         \midrule
         & 96   & 0.189 & 0.258 & \textcolor{red}{\textbf{0.162}} & \textcolor{red}{\textbf{0.217}} & 0.339 & 0.465 & 0.192 & 0.239 & \underline{\textcolor{blue}{0.184}} & 0.246 & 0.205 & 0.241 & 0.232 & 0.271 & 0.217 & 0.255 & 0.194 & 0.255 & 0.200 & \underline{\textcolor{blue}{0.234}} & 0.205 & 0.239 & 0.243 & 0.284 & 0.231 & 0.254   \\
         & 192   & 0.212 & 0.277 & \textcolor{red}{\textbf{0.178}} & \textcolor{red}{\textbf{0.232}} & 0.345 & 0.445 & 0.213 & 0.252 & \underline{\textcolor{blue}{0.201}} & 0.253 & 0.215 & 0.265 & 0.238 & 0.293 & 0.208 & 0.257 & 0.205 & \underline{\textcolor{blue}{0.251}} & 0.239 & 0.294 & 0.227 & 0.280 & 0.253 & 0.311 & 0.257 & 0.301   \\
         & 336   & 0.229 & 0.290 & \textcolor{red}{\textbf{0.187}} & \textcolor{red}{\textbf{0.248}} & 0.431 & 0.470 & 0.222 & 0.261 & 0.220 & 0.263 & \underline{\textcolor{blue}{0.213}} & 0.276 & 0.234 & 0.301 & 0.238 & 0.309 & 0.218 & \underline{\textcolor{blue}{0.257}} & 0.231 & 0.298 & 0.225 & 0.290 & 0.222 & 0.287 & 0.260 & 0.319   \\
       \multirow{-4}*{\rotatebox{90}{Solar}}  & 720   & 0.237 & 0.296 & \textcolor{red}{\textbf{0.195}} & \textcolor{red}{\textbf{0.251}} & 0.241 & 0.308 & 0.235 & \underline{\textcolor{blue}{0.264}} & 0.236 & 0.282 & \underline{\textcolor{blue}{0.232}} & 0.272 & 0.273 & 0.319 & 0.270 & 0.319 & 0.239 & 0.278 & 0.237 & 0.277 & 0.249 & 0.291 & 0.254 & 0.298 & 0.260 & 0.290   \\
         \midrule
         \multicolumn{2}{c|}{$1^{st}$ Count} & \textcolor{red}{\textbf{15}} & \textcolor{red}{\textbf{11}} & \underline{\textcolor{blue}{14}} & 4 & 1 & 4 & 3 & 5 & 0 & 0 & 0 & 1 & 0 & 0 & 0 & 0 & 6 & \underline{\textcolor{blue}{7}} & 0 & 0 & 3 & 3 & 0 & 0 & 0 & 0 \\
         \bottomrule[2pt]
    \end{tabular}
    }
    \caption{Forecasting comparison with lookback $L=720$ and prediction length
$H \in \{96,192,336,720\}$. Baseline results are taken from the
TimeBase~\cite{pmlr-v267-huang25az} paper where available, and otherwise
reproduced using the official code of each baseline. The best results are
shown in \textbf{\textcolor{red}{red}}, second best in
\underline{\textcolor{blue}{blue}}.}
    \label{tab:input_720}
\end{table*}

\paragraph{Phase~2: forecast loss.}
We then optimize the full objective
$\mathcal{L}_{\mathrm{forecast}}=\tfrac12\|\hat{\mathbf{x}}_{\mathrm{f}}-\mathbf{x}_{\mathrm{f}}\|^2$,
reintroducing $T_\mathbf{h}$ and Term~I. Initialized from the Phase~1
solution, $\boldsymbol{\delta}$ already fits the data structure, so
Term~I refines the reconstruction rather than discovering structure under
an ill-conditioned gradient. This schedule improves accuracy by a substantial margin over single-phase
training with
 $\mathcal{L}_{\mathrm{forecast}}$ (Section~\ref{sec:ablation}).

\section{Experiments}\label{sec:exps}

\subsection{Experimental Setup}
\paragraph{Datasets.} To evaluate our proposed method, we conduct experiments on several widely adopted benchmark datasets: the ETT collection (ETTh1, ETTh2, ETTm1, ETTm2)~\cite{informer}, Weather, Traffic, Electricity, and Solar-Energy~\cite{10.1145/3209978.3210006}. Further dataset details are in Appendix~\ref{app:suppl-details}.

\begin{table}[t]
\centering
\scalebox{0.7}{
\begin{tabular}{l|ccc}
\toprule
Dataset & Two-Phase & Single-Phase & Impr.\ (\%) \\
\midrule
ETTh1 & \textcolor{red}{\textbf{0.404}} & 0.458 & \textcolor{red}{\textbf{+11.9}} \\
ETTh2 & \textcolor{red}{\textbf{0.342}} & 0.350 & \textcolor{red}{\textbf{+2.2}} \\
ETTm1 & \textcolor{red}{\textbf{0.350}} & 0.361 & \textcolor{red}{\textbf{+3.0}} \\
ETTm2 & \textcolor{red}{\textbf{0.255}} & 0.262 & \textcolor{red}{\textbf{+2.8}} \\
Weather & \textcolor{red}{\textbf{0.221}} & 0.223 & \textcolor{red}{\textbf{+0.9}} \\
Electricity & \textcolor{red}{\textbf{0.162}} & 0.217 & \textcolor{red}{\textbf{+25.6}} \\
Traffic & \textcolor{red}{\textbf{0.434}} & 0.503 & \textcolor{red}{\textbf{+13.7}} \\
Solar & \textcolor{red}{\textbf{0.253}} & 0.256 & \textcolor{red}{\textbf{+1.2}} \\

\midrule
\textbf{Overall} & & & \textcolor{red}{\textbf{+7.7}} \\
\bottomrule
\end{tabular}}
\caption{Two-phase versus single-phase training, test MSE
averaged over horizons $H\in\{96,192,336,720\}$ at look-back $L=336$.
\textbf{Impr.}\ is the relative MSE reduction of two-phase over
single-phase; positive favors two-phase. Better value in each pair in
\textbf{\textcolor{red}{red}}.}
\label{tab:twophase}
\end{table}

\paragraph{Baselines.}
We compare AdaRDiff against 11 state-of-the-art long-term forecasting baselines, grouped as (i) MLP/Linear-based methods: DLinear~\cite{dlinear}, SparseTSF~\cite{sparse}, TimeMixer~\cite{wang2023timemixer}, CycleNet~\cite{lin2024cyclenet}, TQNet~\cite{lin2025TQNet}, TimeBase~\cite{pmlr-v267-huang25az}, and MixLinear~\cite{ma2026mixlinear}; (ii) Transformer-based methods: FEDformer~\cite{fedformer}, iTransformer~\cite{liu2024itransformer}, and PatchTST~\cite{patchtst}; and (iii) CNN-based methods: TimesNet~\cite{wu2023timesnet}. 

\paragraph{Hyperparameters and Implementation Details.}
Owing to its minimal design, AdaRDiff's only method-specific hyperparameter
is the window size $P$. On the smaller ETT datasets it is tuned over
$\{4,24,96\}$, a grid spanning the dominant daily period of both the hourly
ETTh variants ($24$) and the $15$-minute ETTm variants ($96$), together
with a short window ($4$) capturing only short-range dependence; on the
other datasets $P$ is set to the dominant period (Appendix~\ref{app:suppl-details}). Three further
choices are selected per dataset on the validation set: the initialization
of $\boldsymbol{\delta}$, which is zero ($\boldsymbol{\delta}=\mathbf{0}$),
uniform ($\delta_j=1/P$), or first-order ($\delta_1=1$, else $0$); whether
to apply RevIN~\cite{revin} for distribution
shift; and whether to apply the $\ell_1$-reparameterization of
Eq.~\eqref{eq:reparam}. In the main results we evaluate linear and MLP
backbones, with all hyperparameters tuned on the validation set. Models are
implemented in PyTorch~\cite{10.5555/3454287.3455008} and trained with
Adam~\cite{Kingma2014AdamAM} for $30$ epochs in Phase~1 and $10$ in
Phase~2, with early stopping. We report MSE and MAE, and run all experiments
on a single NVIDIA H100 (80\,GB). Full details are in Appendix~\ref{app:pseudocode}.

\subsection{Main Results}
Table~\ref{tab:input_720} reports the full forecasting
comparison.\footnote{Results for shorter look-back windows
($L\in\{96,336,512\}$) are reported in Appendix~\ref{app:suppl-details}.} Across the eight
benchmarks, the AdaRDiff-Linear backbone attains the best results more often than the other methods, on both MSE and MAE (on $15$ and $11$ cases respectively), with the MLP backbone being second on MSE ($14$); together the two variants dominate the count of top-ranked scores by a wide margin over every competing method. The gains are largest on Solar-Energy, where the MLP backbone lowers MSE by
$11$--$16\%$ over the strongest baseline across all horizons. Photovoltaic output alternates zero-valued nights with gradual daytime
rises and falls \cite{lin2024cyclenet}, so consecutive values are highly predictable from
the immediate past: short-range structure rather than a seasonal cycle.
AdaRDiff suits this, its differencing weights concentrating at the smallest lags rather
than the daily period (Figure \ref{fig:delta_periodicity}), acting as a
low-order differencing.
On the high-dimensional Traffic dataset ($862$ channels), AdaRDiff matches the strong PatchTST baseline at a fraction (ca.\ $6.6$\%) of its parameter count, and clearly outperforms all other methods. We attribute this to the strong temporal-lag structure of Traffic \cite{lin2024cyclenet}: by conditioning each prediction on recent lagged values, AdaRDiff's recursive reconstruction directly exploits the short-range dependencies that dominate this dataset. Notably, all of these results are obtained with a lightweight linear or MLP backbone, indicating that the learnable differencing operator, rather than backbone capacity, drives the improvement.
\subsection{Ablation Study and Model Analysis}
\label{sec:ablation}

\paragraph{Two-phase training.}
Table~\ref{tab:twophase} compares the two-phase schedule with
single-phase training under identical settings. Two-phase training lowers
test MSE on all eight datasets, by $7.7\%$ on average, with the largest
gains on the highest-dimensional datasets: Electricity ($+25.6\%$, $321$
channels) and Traffic ($+13.7\%$, $862$ channels). Since a distinct
$\boldsymbol{\delta}$ is learned per channel, these are the settings in
which single-phase optimization must resolve the tension between the two
pathways across hundreds of independent $\boldsymbol{\delta}$, exactly where
separating the two phases helps most.


\paragraph{AdaRDiff as a Plug-and-Play Module.}
\begin{table}[t]
\centering

\renewcommand{\arraystretch}{1.1}
\setlength{\tabcolsep}{4pt}
\scalebox{0.5}{
\begin{tabular}{@{}l|l|c|c|c|c|c|c|c|c@{}}
\toprule
\multicolumn{2}{c|}{Models} & \multicolumn{1}{c|}{Linear} & \multicolumn{1}{c|}{DLinear} & \multicolumn{1}{c|}{MLP} & \multicolumn{1}{c|}{MixLinear} & \multicolumn{1}{c|}{SparseTSF} & \multicolumn{1}{c|}{CycleNet} & \multicolumn{1}{c|}{PatchTST} & \multicolumn{1}{c}{iTransformer} \\ \midrule
\multirow{3}{*}{ETTh1} & Original & 0.441 & 0.425 & 0.428 & 0.401 & 0.420 & 0.417 & 0.419 & 0.438 \\
 & \textbf{+AdaRDiff} & 0.404 & 0.409 & 0.425 & 0.400 & 0.402 & 0.408 & 0.413 & 0.427 \\ \cmidrule(lr){2-10}
 & Promotion & \textbf{8.4\%} & \textbf{3.8\%} & \textbf{0.6\%} & \textbf{0.4\%} & \textbf{4.4\%} & \textbf{2.0\%} & \textbf{1.3\%} & \textbf{2.4\%} \\ \addlinespace\cline{1-10} \addlinespace
\multirow{3}{*}{ETTh2} & Original & 0.462 & 0.470 & 0.371 & 0.355 & 0.350 & 0.351 & 0.351 & 0.396 \\
 & \textbf{+AdaRDiff} & 0.342 & 0.355 & 0.345 & 0.344 & 0.338 & 0.355 & 0.347 & 0.361 \\ \cmidrule(lr){2-10}
 & Promotion & \textbf{25.9\%} & \textbf{24.4\%} & \textbf{7.1\%} & \textbf{3.2\%} & \textbf{3.6\%} & \color{red}-\textbf{1.1\%} & \textbf{1.1\%} & \textbf{8.8\%} \\ \addlinespace\cline{1-10} \addlinespace
\multirow{3}{*}{ETTm1} & Original & 0.362 & 0.363 & 0.357 & 0.379 & 0.374 & 0.355 & 0.351 & 0.365 \\
 & \textbf{+AdaRDiff} & 0.350 & 0.353 & 0.349 & 0.356 & 0.356 & 0.350 & 0.342 & 0.367 \\ \cmidrule(lr){2-10}
 & Promotion & \textbf{3.3\%} & \textbf{2.8\%} & \textbf{2.1\%} & \textbf{6.2\%} & \textbf{5.0\%} & \textbf{1.6\%} & \textbf{2.5\%} & \color{red}-\textbf{0.4\%} \\ \addlinespace\cline{1-10} \addlinespace
\multirow{3}{*}{ETTm2} & Original & 0.293 & 0.274 & 0.268 & 0.267 & 0.266 & 0.252 & 0.258 & 0.281 \\
 & \textbf{+AdaRDiff} & 0.255 & 0.256 & 0.253 & 0.258 & 0.258 & 0.256 & 0.258 & 0.270 \\ \cmidrule(lr){2-10}
 & Promotion & \textbf{13.1\%} & \textbf{6.6\%} & \textbf{5.9\%} & \textbf{3.4\%} & \textbf{3.1\%} & \color{red}-\textbf{1.5\%} & \color{red}-\textbf{0.2\%} & \textbf{3.8\%} \\ \addlinespace\cline{1-10} \addlinespace
\multirow{3}{*}{Weather} & Original & 0.245 & 0.245 & 0.236 & 0.257 & 0.257 & 0.243 & 0.230 & 0.235 \\
 & \textbf{+AdaRDiff} & 0.221 & 0.225 & 0.221 & 0.230 & 0.226 & 0.237 & 0.225 & 0.230 \\ \cmidrule(lr){2-10}
 & Promotion & \textbf{9.9\%} & \textbf{8.1\%} & \textbf{6.5\%} & \textbf{10.5\%} & \textbf{11.8\%} & \textbf{2.5\%} & \textbf{2.3\%} & \textbf{1.8\%} \\ \addlinespace\cline{1-10} \addlinespace

\multirow{3}{*}{Electricity} & Original & 0.177 & 0.168 & 0.165 & 0.176 & 0.173 & 0.160 & 0.169 & 0.203 \\
 & \textbf{+AdaRDiff} & 0.162 & 0.165 & 0.162 & 0.168 & 0.167 & 0.159 & 0.163 & 0.166 \\ \cmidrule(lr){2-10}
 & Promotion & \textbf{8.4\%} & \textbf{1.9\%} & \textbf{1.6\%} & \textbf{4.6\%} & \textbf{3.4\%} & \textbf{1.0\%} & \textbf{3.4\%} & \textbf{18.3\%} \\ \addlinespace\cline{1-10} \addlinespace
\multirow{3}{*}{Traffic} & Original & 0.450 & 0.436 & 0.408 & 0.441 & 0.438 & 0.423 & 0.413 & 0.528 \\
 & \textbf{+AdaRDiff} & 0.434 & 0.435 & 0.406 & 0.439 & 0.437 & 0.421 & 0.422 & 0.438 \\ \cmidrule(lr){2-10}
 & Promotion & \textbf{3.6\%} & \textbf{0.3\%} & \textbf{0.3\%} & \textbf{0.4\%} & \textbf{0.0\%} & \textbf{0.5\%} & \color{red}-\textbf{2.2\%} & \textbf{17.1\%} \\ \addlinespace\cline{1-10} \addlinespace
 \multirow{3}{*}{Solar} & Original & 0.258 & 0.253 & 0.217 & 0.280 & 0.278 & 0.261 & 0.214 & 0.238 \\
 & \textbf{+AdaRDiff} & 0.253 & 0.262 & 0.191 & 0.269 & 0.266 & 0.249 & 0.200 & 0.207 \\ \cmidrule(lr){2-10}
 & Promotion & \textbf{2.0\%} & \color{red}-\textbf{3.5\%} & \textbf{11.7\%} & \textbf{3.7\%} & \textbf{4.3\%} & \textbf{4.4\%} & \textbf{6.5\%} & \textbf{13.1\%} \\

\bottomrule
\end{tabular}
}
\caption{Effect of \textbf{AdaRDiff} on each backbone (MSE), averaged over
horizons at $L=336$. \textbf{Promotion} is the relative MSE reduction
(\textcolor{red}{red} = regression).}
\label{tab:adardiff_avg_mse}
\end{table}
Table~\ref{tab:adardiff_avg_mse} evaluates AdaRDiff as a plug-and-play
module across eight backbones spanning linear models, MLPs, and
Transformers. AdaRDiff improves MSE in the large majority of
configurations ($56$ of $64$ backbone--dataset pairs), and no regression
exceeds $3.5\%$. Three patterns emerge. First, the gains are largest on backbones that lack any dedicated
mechanism for temporal structure. The plain linear model and MLP improve
by up to $25.9\%$ (linear, ETTh2), and a simple linear backbone equipped
with AdaRDiff matches or surpasses
much larger Transformer baselines---%
evidence that the improvement originates in the adaptive differencing and
reconstruction rather than in backbone capacity. Second, and more revealing, AdaRDiff also improves backbones that already
model temporal structure explicitly. DLinear removes trend by moving
average, SparseTSF and CycleNet model fixed periodic patterns, and
MixLinear separates trend from low-rank spectral components; each captures
a single form of structure in isolation. In contrast, AdaRDiff models
trend and seasonality jointly through one learnable differencing operator,
and therefore remains complementary even to backbones with built-in
decomposition, improving SparseTSF and MixLinear by up to $11.8\%$ and
$10.5\%$ respectively on Weather. Third, the benefit is most pronounced on iTransformer, which models
cross-variate correlations rather than per-series temporal dynamics. Here
AdaRDiff's per-channel differencing supplies the temporal component the
backbone lacks, yielding large gains on the high-dimensional Electricity
($+18.3\%$) and Traffic ($+17.1\%$) datasets. That AdaRDiff strengthens
even these specialized backbones establishes it as an effective and
broadly portable forecasting component.

\begin{figure}[t]
    \centering
    \includegraphics[width=0.4\textwidth]{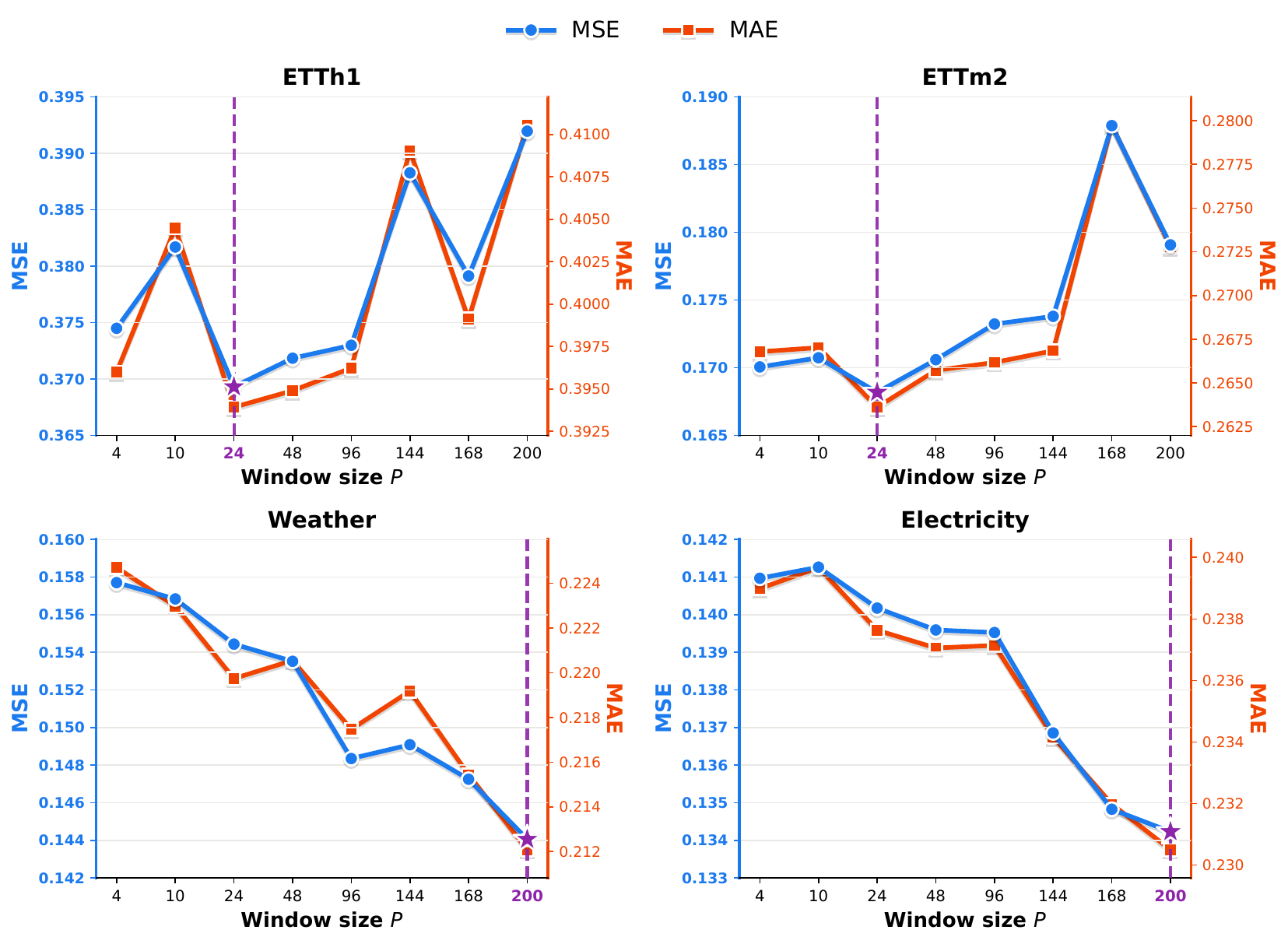}
    \caption{Performance of the AdaRDiff/Linear model with varied $P$.
The forecast horizon is set as 96. The dashed line marks the window size minimizing
 MSE.}
    \label{fig:window_size}
\end{figure}

\begin{figure*}[t]
    \begin{subfigure}[t]{0.7\textwidth}
        \includegraphics[width=\textwidth]{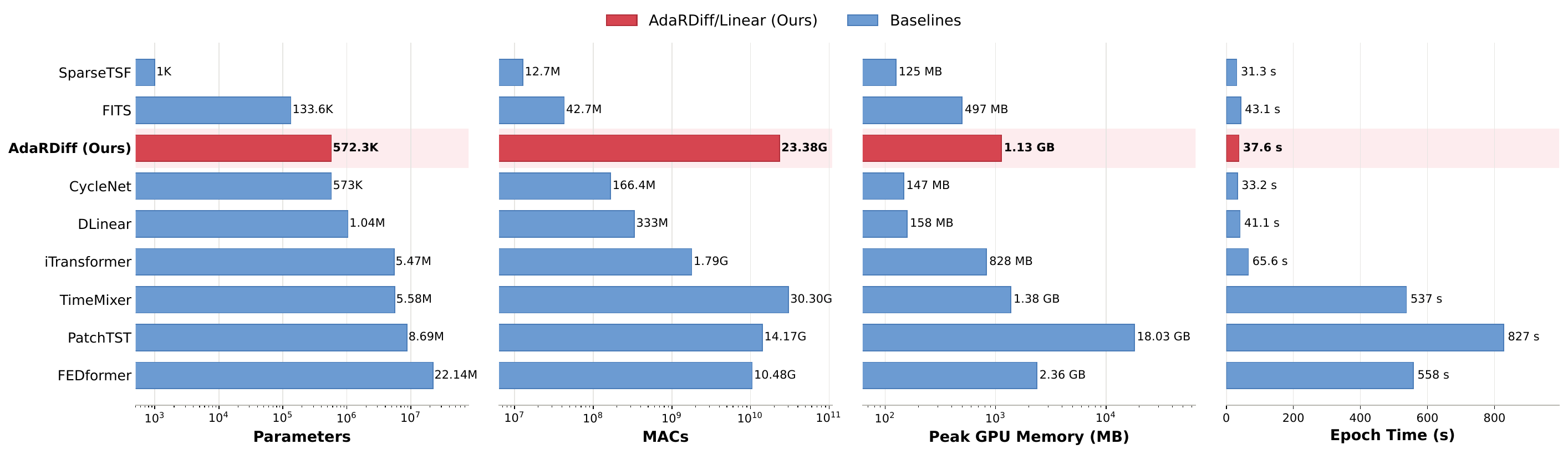}
        \caption{Compute, time, and memory}
        \label{fig:efficiency_combined}
    \end{subfigure}
    \begin{subfigure}[t]{0.25\textwidth}
        \includegraphics[width=\textwidth]{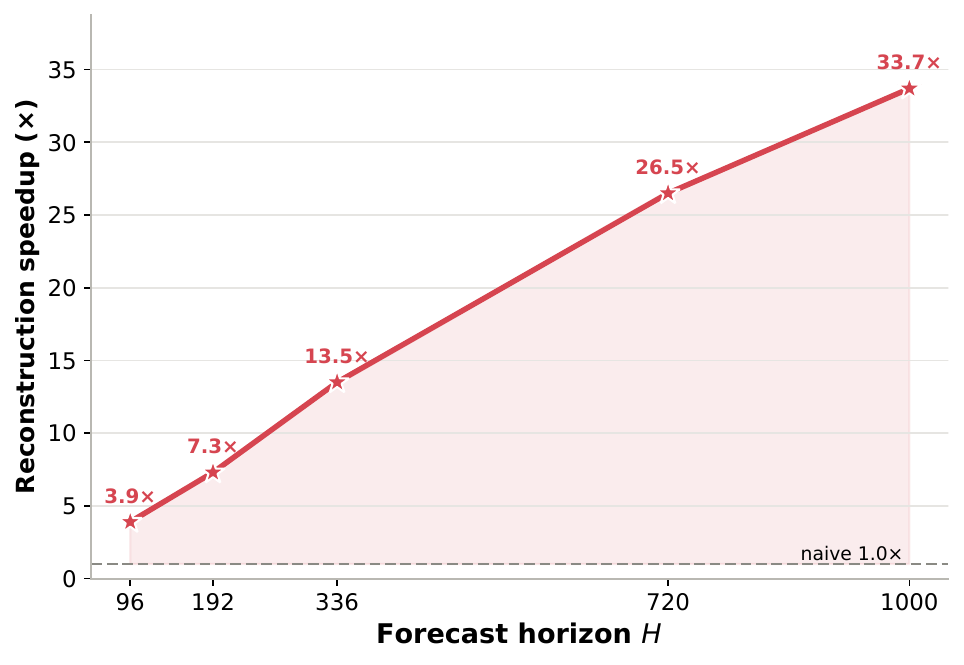}
        \caption{Reconstruction speedup}
        \label{fig:efficiency_speedup}
    \end{subfigure}
    \caption{Efficiency of AdaRDiff/Linear.
    \textbf{(a)} Number of parameters, compute (MACs per sample), peak GPU memory, and epoch time across models on Electricity (forecast length and look-back window
    both set to $720$; epoch time and memory measured at batch size $12$).
    \textbf{(b)} Inference speedup of the closed-form convolutional
    reconstruction over the naive sequential reconstruction as a function of
    the forecast horizon $H$.}
    \label{fig:efficiency}
\end{figure*}

\paragraph{Effect of the window size $P$.} Figure~\ref{fig:window_size} reports MSE against the differencing window
$P$ on four datasets of differing periods. The best $P$ broadly
reflects the temporal structure of each dataset: ETTh1, with a daily period
($24$ steps), is optimal at $P=24$, while the longer-period,
high-dimensional datasets Weather and Electricity (periods of $144$ and
$168$) favor larger windows. On ETTm2, MSE changes little between $P=24$
and its dominant period $P=96$, showing that AdaRDiff is robust to this
choice. Overall, performance tracks the dominant period while remaining
stable to moderate variations in $P$.

\begin{figure}[!t]
   
    \includegraphics[width=0.45\textwidth]{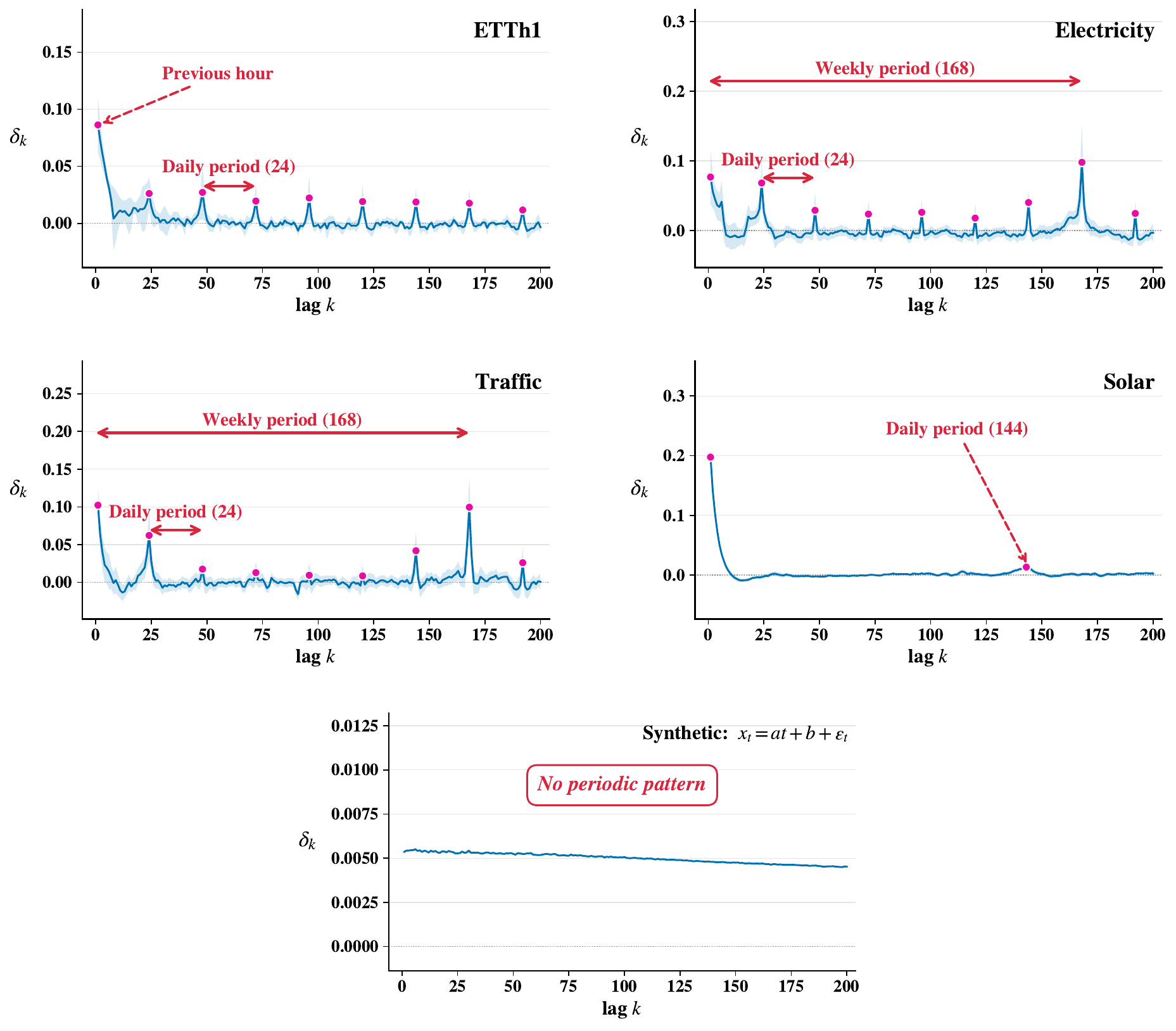}
    \caption{Learned differencing weights $\delta_k$ across lags for
    AdaRDiff with look-back $L=336$, forecast horizon
    $H=96$, window size extended to $P=200$ and $\boldsymbol{\delta}$
    initialized to zero). Curves show the mean over channels, with shaded regions denoting $\pm$ one standard
deviation.}
    \label{fig:delta_periodicity}
\end{figure}

\paragraph{Efficiency analysis.}

Figure~\ref{fig:efficiency}(a) reports the number of parameters, the compute (MACs per sample), the peak GPU memory (in MB) and the average epoch time (in seconds) on
Electricity. As a plug-and-play module, AdaRDiff adds only $C\!\times\!P$
parameters ($572$K with a linear backbone, comparable to CycleNet and far
below PatchTST at $8.7$M), and its epoch time ($37.6$\,s) matches the
lightweight linear models. The higher MACs and memory come from the
reconstruction, which builds the sequence $\mathbf{h}$ in
$O(C\,P^{3}\log H)$; the cubic dependence on $P$ dominates on long-period
datasets ($P{=}168$ on Electricity). Since accuracy remains stable to
moderate variations in $P$ (Figure~\ref{fig:window_size}), reducing $P$
saves compute at negligible accuracy loss: on Electricity, $P{=}48$
vs.\ $168$ cuts the reconstruction cost $\sim\!43\times$ for only a
$2.9\%$ accuracy loss. Crucially, although the raw MAC count is high, our
reconstruction replaces the sequential recurrence with a parallel
formulation of $O(\log H)$ which runs efficiently on GPU. This explains why the epoch time is low despite the peak memory usage and compute values. Figure~\ref{fig:efficiency}(b) shows that our reconstruction is up to $33.7\times$ faster than the naive recurrence, the gain growing with the horizon. \emph{AdaRDiff thus attains its accuracy gains at a compute cost which is partly controllable through $P$, and a runtime competitive with the lightest baselines.}

\paragraph{Visualization and Interpretation of the Learned Differencing Weights.}AdaRDiff learns the differencing weights $\boldsymbol{\delta}$ directly from
the training data. Figure~\ref{fig:delta_periodicity} shows the learned
$\delta_k$ across lags. On ETTh1, Electricity, and Traffic the weights
exhibit clear seasonal structure: sharp peaks at the daily period ($24$) on
ETTh1, and at both the daily ($24$) and weekly ($168$) periods on
Electricity and Traffic, aligning with the known period lengths of each
dataset (Appendix~\ref{app:suppl-details}). On ETTh1 the coefficient at lag~$1$ is also large and
decays quickly, indicating strong short-range dependence on the immediate
past. On a
synthetic series with a linear trend and no periodicity
($x_t = at + b + \varepsilon_t$), by contrast, the weights are
near-uniform ($\delta_k \approx 1/P = 0.005$) with no periodic peaks;
applied as a differencing operator, such weights subtract the moving
average of the recent past, removing the linear trend---leaving a
stable residual---while averaging out the noise. The model thus
produces distinct weight patterns for distinct temporal structures, adapting its
differencing to each series without human hard-coding and providing an
interpretable account of its behavior.

\section{Conclusion}
We introduced AdaRDiff, a model-agnostic module that revisits classical
differencing through a learnable, reversible operator. We provided both a theoretical analysis of how AdaRDiff behaves with a linear backbone, as well as an experimental validation including different backbones.
Across eight benchmarks and eight backbones, AdaRDiff attains
state-of-the-art accuracy, yields consistent plug-and-play
gains, and produces interpretable differencing weights that recover 
temporal structures. More broadly, our results show that classical forecasting tools offer
principled inductive biases that make deep models more robust and
interpretable when adapted to end-to-end learning. A compelling next step is
to bring this perspective to time series foundation models: by handling
trend and seasonality explicitly, a learnable differencing operator addresses
the dominant temporal structure directly, which could let these models reach
strong accuracy with far fewer parameters while exposing, through the learned
weights, the temporal patterns they otherwise capture only implicitly.

\bibliography{aaai2027}

@inproceedings{
liu2024itransformer,
title={iTransformer: Inverted Transformers Are Effective for Time Series Forecasting},
author={Yong Liu and Tengge Hu and Haoran Zhang and Haixu Wu and Shiyu Wang and Lintao Ma and Mingsheng Long},
booktitle={The Twelfth International Conference on Learning Representations},
year={2024},
url={https://openreview.net/forum?id=JePfAI8fah}
}

@inproceedings{
zhang2023crossformer,
title={Crossformer: Transformer Utilizing Cross-Dimension Dependency for Multivariate Time Series Forecasting},
author={Yunhao Zhang and Junchi Yan},
booktitle={The Eleventh International Conference on Learning Representations },
year={2023},
url={https://openreview.net/forum?id=vSVLM2j9eie}
}

@inproceedings{
liu2022nonstationary,
title={Non-stationary Transformers: Exploring the Stationarity in Time Series Forecasting},
author={Yong Liu and Haixu Wu and Jianmin Wang and Mingsheng Long},
booktitle={Advances in Neural Information Processing Systems},
editor={Alice H. Oh and Alekh Agarwal and Danielle Belgrave and Kyunghyun Cho},
year={2022},
url={https://openreview.net/forum?id=ucNDIDRNjjv}
}

@InProceedings{fedformer,
  title = 	 {{FED}former: Frequency Enhanced Decomposed Transformer for Long-term Series Forecasting},
  author =       {Zhou, Tian and Ma, Ziqing and Wen, Qingsong and Wang, Xue and Sun, Liang and Jin, Rong},
  booktitle = 	 {Proceedings of the 39th International Conference on Machine Learning},
  pages = 	 {27268--27286},
  year = 	 {2022},
  editor = 	 {Chaudhuri, Kamalika and Jegelka, Stefanie and Song, Le and Szepesvari, Csaba and Niu, Gang and Sabato, Sivan},
  volume = 	 {162},
  series = 	 {Proceedings of Machine Learning Research},
  month = 	 {17--23 Jul},
  publisher =    {PMLR},
  url = 	 {https://proceedings.mlr.press/v162/zhou22g.html}
}

@inproceedings{10.5555/3692070.3692474,
author = {Das, Abhimanyu and Kong, Weihao and Sen, Rajat and Zhou, Yichen},
title = {A decoder-only foundation model for time-series forecasting},
year = {2024},
publisher = {JMLR.org},
booktitle = {Proceedings of the 41st International Conference on Machine Learning},
articleno = {404},
numpages = {20},
location = {Vienna, Austria},
series = {ICML'24}
}

@article{ansari2024chronos,
  title={Chronos: Learning the Language of Time Series},
  author={Ansari, Abdul Fatir and Stella, Lorenzo and Turkmen, Caner and Zhang, Xiyuan and Mercado, Pedro and Shen, Huibin and Shchur, Oleksandr and Rangapuram, Syama Syndar and Pineda Arango, Sebastian and Kapoor, Shubham and Zschiegner, Jasper and Maddix, Danielle C. and Mahoney, Michael W. and Torkkola, Kari and Gordon Wilson, Andrew and Bohlke-Schneider, Michael and Wang, Yuyang},
  journal={Transactions on Machine Learning Research},
  issn={2835-8856},
  year={2024},
  url={https://openreview.net/forum?id=gerNCVqqtR}
}

@inproceedings{niu2026phaseformer,
  title={PhaseFormer: From Patches to Phases for Efficient and Effective Time Series Forecasting},
  author={Niu, Yiming and Deng, Jinliang and Tong, Yongxin},
  booktitle={International Conference on Learning Representations (ICLR)},
  year={2026}
}

@misc{laglil2026foundationmodelsfinetuningnew,
      title={Foundation Models and Fine-Tuning: Toward a New Generation of Models for Time Series Forecasting}, 
      author={Morad Laglil and Bertrand Pracca and Emilie Devijver and Eric Gaussier},
      year={2026},
      eprint={2607.23146},
      archivePrefix={arXiv},
      primaryClass={cs.LG},
      url={https://arxiv.org/abs/2607.23146}, 
}

@article{qiu2024tfb,
  title   = {TFB: Towards Comprehensive and Fair Benchmarking of Time Series Forecasting Methods},
  author  = {Xiangfei Qiu and Jilin Hu and Lekui Zhou and Xingjian Wu and Junyang Du and Buang Zhang and Chenjuan Guo and Aoying Zhou and Christian S. Jensen and Zhenli Sheng and Bin Yang},
  journal = {Proc. {VLDB} Endow.},
  volume  = {17},
  number  = {9},
  pages   = {2363--2377},
  year    = {2024}
}

@inproceedings{10.1145/3209978.3210006,
author = {Lai, Guokun and Chang, Wei-Cheng and Yang, Yiming and Liu, Hanxiao},
title = {Modeling Long- and Short-Term Temporal Patterns with Deep Neural Networks},
year = {2018},
isbn = {9781450356572},
publisher = {Association for Computing Machinery},
address = {New York, NY, USA},
url = {https://doi.org/10.1145/3209978.3210006},
doi = {10.1145/3209978.3210006},
booktitle = {The 41st International ACM SIGIR Conference on Research \& Development in Information Retrieval},
pages = {95–104},
numpages = {10},
location = {Ann Arbor, MI, USA},
series = {SIGIR '18}
}

@inproceedings{patchtst,
  title     = {A Time Series is Worth 64 Words: Long-term Forecasting with Transformers},
  author    = {Nie, Yuqi and
               H. Nguyen, Nam and
               Sinthong, Phanwadee and 
               Kalagnanam, Jayant},
  booktitle = {International Conference on Learning Representations},
  year      = {2023}
}

@inproceedings{informer,
  author       = {Haoyi Zhou and
                  Shanghang Zhang and
                  Jieqi Peng and
                  Shuai Zhang and
                  Jianxin Li and
                  Hui Xiong and
                  Wancai Zhang},
  title        = {Informer: Beyond Efficient Transformer for Long Sequence Time-Series
                  Forecasting},
  booktitle    = {Thirty-Fifth {AAAI} Conference on Artificial Intelligence, {AAAI}
                  2021, Thirty-Third Conference on Innovative Applications of Artificial
                  Intelligence, {IAAI} 2021, The Eleventh Symposium on Educational Advances
                  in Artificial Intelligence, {EAAI} 2021,  February 2-9,
                  2021},
  pages        = {11106--11115},
  publisher    = {{AAAI} Press},
  year         = {2021},
  url          = {https://doi.org/10.1609/aaai.v35i12.17325},
  doi          = {10.1609/AAAI.V35I12.17325},
  bibsource    = {dblp computer science bibliography, https://dblp.org},
  address= {Virtual Event}
}

@inproceedings{10.1609/aaai.v39i11.33267,
author = {Fei, Jingru and Yi, Kun and Fan, Wei and Zhang, Qi and Niu, Zhendong},
title = {Amplifier: bringing attention to neglected low-energy components in time series forecasting},
year = {2025},
isbn = {978-1-57735-897-8},
publisher = {AAAI Press},
doi = {10.1609/aaai.v39i11.33267},
booktitle = {Proceedings of the Thirty-Ninth AAAI Conference on Artificial Intelligence and Thirty-Seventh Conference on Innovative Applications of Artificial Intelligence and Fifteenth Symposium on Educational Advances in Artificial Intelligence},
articleno = {1294},
numpages = {9},
series = {AAAI'25/IAAI'25/EAAI'25}
}

@inproceedings{sparse,
author = {Lin, Shengsheng and Lin, Weiwei and Wu, Wentai and Chen, Haojun and Yang, Junjie},
title = {SparseTSF: modeling long-term time series forecasting with 1k parameters},
year = {2024},
publisher = {JMLR.org},
booktitle = {Proceedings of the 41st International Conference on Machine Learning},
articleno = {1216},
numpages = {16},
location = {Vienna, Austria},
series = {ICML'24}
}

@book{OTexts,
    title = "Forecasting: Principles and Practice",
    author = "Hyndman, Robin John and  Athanasopoulos, George",
    year = "2021",
    language = "English",
    publisher = "OTexts",
    address = "Australia",
    edition = "3rd",

}

@article{transformer,
  title={Attention is all you need},
  author={Vaswani, Ashish and Shazeer, Noam and Parmar, Niki and Uszkoreit, Jakob and Jones, Llion and Gomez, Aidan N and Kaiser, {\L}ukasz and Polosukhin, Illia},
  journal={Advances in neural information processing systems},
  volume={30},
  year={2017}
}

@inproceedings{wang2023timemixer,
  title={TimeMixer: Decomposable Multiscale Mixing for Time Series Forecasting},
  author={Wang, Shiyu and Wu, Haixu and Shi, Xiaoming and Hu, Tengge and Luo, Huakun and Ma, Lintao and Zhang, James Y and ZHOU, JUN},
  booktitle={International Conference on Learning Representations (ICLR)},
  year={2024}
}

@inproceedings{dlinear,
  title     = {Are Transformers Effective for Time Series Forecasting?},
  author    = {Zeng, Ailing and Chen, Muxi and Zhang, Lei and Xu, Qiang},
  booktitle = {Proceedings of the AAAI Conference on Artificial Intelligence},
  year      = {2023}
}

@inproceedings{10.5555/3454287.3455008,
  author    = {Paszke, Adam and Gross, Sam and Massa, Francisco and Lerer, Adam and Bradbury, James and Chanan, Gregory and Killeen, Trevor and Lin, Zeming and Gimelshein, Natalia and Antiga, Luca and Desmaison, Alban and K\"{o}pf, Andreas and Yang, Edward and DeVito, Zach and Raison, Martin and Tejani, Alykhan and Chilamkurthy, Sasank and Steiner, Benoit and Fang, Lu and Bai, Junjie and Chintala, Soumith},
  title     = {PyTorch: An Imperative Style, High-Performance Deep Learning Library},
  booktitle = {Advances in Neural Information Processing Systems (NeurIPS)},
  pages     = {8026--8037},
  year      = {2019}
}

@article{Kingma2014AdamAM,
  title={Adam: A Method for Stochastic Optimization},
  author={Diederik P. Kingma and Jimmy Ba},
  journal={CoRR},
  year={2014},
  volume={abs/1412.6980},
  url={https://api.semanticscholar.org/CorpusID:6628106}
}

@book{knuth1997taocp2,
  author    = {Knuth, Donald E.},
  title     = {The Art of Computer Programming, Volume 2:
               Seminumerical Algorithms},
  edition   = {3rd},
  publisher = {Addison-Wesley},
  year      = {1997}
}

@book{elaydi2005difference,
  author    = {Saber Elaydi},
  title     = {An Introduction to Difference Equations},
  edition   = {3},
  year      = {2005},
  publisher = {Springer},
  address   = {New York, NY},
  series    = {Undergraduate Texts in Mathematics},
  isbn      = {978-0-387-27602-1},
  doi       = {10.1007/0-387-27602-5}
}

@inproceedings{wu2023timesnet,
  title={TimesNet: Temporal 2D-Variation Modeling for General Time Series Analysis},
  author={Haixu Wu and Tengge Hu and Yong Liu and Hang Zhou and Jianmin Wang and Mingsheng Long},
  booktitle={International Conference on Learning Representations},
  year={2023},
}

@inproceedings{
    lin2024cyclenet,
    title={CycleNet: Enhancing Time Series Forecasting through Modeling Periodic Patterns},
    author={Shengsheng Lin and Weiwei Lin and Xinyi HU and Wentai Wu and Ruichao Mo and Haocheng Zhong},
    booktitle={The Thirty-eighth Annual Conference on Neural Information Processing Systems},
    year={2024},
}

@book{box1976time,
  title={Time Series Analysis: Forecasting and Control},
  author={Box, G.E.P. and Jenkins, G.M.},
  isbn={9780816211043},
  lccn={lc76008713},
  series={Holden-Day series in time series analysis and digital processing},
  url={https://books.google.fr/books?id=1WVHAAAAMAAJ},
  year={1976},
  publisher={Holden-Day}
}

@article{rb1990stl,
  title={STL: A seasonal-trend decomposition procedure based on loess},
  author={Cleveland, Robert B and Cleveland, William S and Terpenning, Irma},
  journal={Journal of Official Statistics},
  volume={6},
  pages={3--73},
  year={1990}
}

@article{NNFSTS,
    author = {Zhang, Peter and Qi, Min},
    year = {2005},
    month = {02},
    pages = {501-514},
    title = {Neural network forecasting for seasonal and trend time series},
    volume = {160},
    journal = {European Journal of Operational Research},
    doi = {10.1016/j.ejor.2003.08.037}
    }

@inproceedings{revin,
  title={Reversible instance normalization for accurate time-series forecasting against distribution shift},
  author={Kim, Taesung and Kim, Jinhee and Tae, Yunwon and Park, Cheonbok and Choi, Jang-Ho and Choo, Jaegul},
  booktitle={International Conference on Learning Representations},
  year={2021}
}

@inproceedings{lin2025TQNet,
  title={Temporal Query Network for Efficient Multivariate Time Series Forecasting}, 
  author={Lin, Shengsheng and Chen, Haojun and Wu, Haijie and Qiu, Chunyun and Lin, Weiwei},
  booktitle={Forty-second International Conference on Machine Learning},
  year={2025}
}

@InProceedings{pmlr-v267-huang25az,
  title = 	 {{T}ime{B}ase: The Power of Minimalism in Efficient Long-term Time Series Forecasting},
  author =       {Huang, Qihe and Zhou, Zhengyang and Yang, Kuo and Yi, Zhongchao and Wang, Xu and Wang, Yang},
  booktitle = 	 {Proceedings of the 42nd International Conference on Machine Learning},
  pages = 	 {26227--26246},
  year = 	 {2025},
  editor = 	 {Singh, Aarti and Fazel, Maryam and Hsu, Daniel and Lacoste-Julien, Simon and Berkenkamp, Felix and Maharaj, Tegan and Wagstaff, Kiri and Zhu, Jerry},
  volume = 	 {267},
  series = 	 {Proceedings of Machine Learning Research},
  month = 	 {13--19 Jul},
  publisher =    {PMLR},
  url = 	 {https://proceedings.mlr.press/v267/huang25az.html}
}

@inproceedings{ma2026mixlinear,
  title={MixLinear: Extreme Low-Resource Multivariate Time Series Forecasting with 0.1K Parameters},
  author={Ma, Aitian and Luo, Dongsheng and Sha, Mo},
  booktitle={International Conference on Learning Representations (ICLR)},
  year={2026}
}


\appendix

\section{Theoretical Proofs}

\subsection{Proof of Proposition~\ref{prop:convolution}}

Write $v_k:=\hat{z}_k+c_k$, so that the identity to prove,
Eq.~\eqref{eq:convolution}, reads $\hat{x}_{L+k}=\sum_{m=0}^{k-1}h_m v_{k-m}$.
Throughout we use the conventions $\delta_j=0$ for $j>P$ and $h_n=0$ for
$n<0$, and argue by strong induction on $k$.

\paragraph{Base case ($k=1$).}
Since $L\ge P$, every term $\hat{x}_{L+1-j}$ in the reconstruction
\eqref{eq:reverse} is an observed value $x_{L+1-j}$. Hence
\[
\hat{x}_{L+1}=\hat{z}_1+\sum_{j=1}^{P}\delta_j x_{L+1-j}
=\hat{z}_1+c_1=v_1=h_0 v_1,
\]
using $c_1=\sum_{j=1}^{P}\delta_j x_{L+1-j}$ and $h_0=1$, which is
Eq.~\eqref{eq:convolution} at $k=1$.

\paragraph{Inductive step.}
Fix $k\ge2$ and assume Eq.~\eqref{eq:convolution} for all indices
$1,\ldots,k-1$. In \eqref{eq:reverse}, the term $\hat{x}_{L+k-j}$ is a
previously reconstructed forecast when $k-j\ge1$ and an observed value
$x_{L+k-j}$ when $k-j\le0$. Splitting the sum at $j=k$, the observed part
is $\sum_{j=k}^{P}\delta_j x_{L+k-j}=c_k$, so
\[
\hat{x}_{L+k}=v_k+\sum_{j=1}^{k-1}\delta_j\,\hat{x}_{L+k-j}.
\]
Each remaining index satisfies $k-j\in\{1,\ldots,k-1\}$, so the induction
hypothesis gives $\hat{x}_{L+k-j}=\sum_{i=0}^{k-j-1}h_i v_{k-j-i}$.
Substituting and reindexing by $m=i+j$ (so $m$ runs over $1,\ldots,k-1$
and, for each $m$, $j$ over $1,\ldots,m$),
\[
\begin{aligned}
\sum_{j=1}^{k-1}\delta_j\sum_{i=0}^{k-j-1}h_i v_{k-j-i}
&=\sum_{m=1}^{k-1}\Bigl(\sum_{j=1}^{m}\delta_j h_{m-j}\Bigr)v_{k-m}\\
&=\sum_{m=1}^{k-1}h_m v_{k-m},
\end{aligned}
\]
the inner sum equaling $h_m$ by the recurrence defining $\mathbf{h}$. With
$h_0=1$,
\[
\hat{x}_{L+k}=v_k+\sum_{m=1}^{k-1}h_m v_{k-m}=\sum_{m=0}^{k-1}h_m v_{k-m},
\]
which is Eq.~\eqref{eq:convolution} at index $k$ and completes the
induction. \hfill$\qed$

\subsection{Proof of Proposition~\ref{prop:stability}}

\paragraph{Setup.}
The sequence $(h_n)$ with $h_0=1$, $h_n=0$ for $n<0$, and
$h_n=\sum_{j=1}^{P}\delta_j h_{n-j}$ ($n\ge1$) is the fundamental solution
of a $P$th-order constant-coefficient linear recurrence, with
characteristic polynomial
$\chi(\lambda)=\lambda^{P}-\sum_{j=1}^{P}\delta_j\lambda^{P-j}$. Let
$\lambda_1,\ldots,\lambda_r$ be its distinct roots, $m_1,\ldots,m_r$ their
multiplicities, and $\rho=\max_i|\lambda_i|$; these are the eigenvalues of
the companion matrix of $\boldsymbol{\delta}$. By the general solution of
such recurrences \cite[Cor.~2.24]{elaydi2005difference},
\begin{equation}
h_n=\sum_{i=1}^{r}p_i(n)\,\lambda_i^{\,n},
\qquad \deg p_i\le m_i-1 .
\label{eq:pf_gensol}
\end{equation}

\paragraph{(i) Case $\rho<1$.}
Fix $\rho'\in(\rho,1)$. Since $|\lambda_i|\le\rho$, each term of
\eqref{eq:pf_gensol} obeys $|p_i(n)\lambda_i^{\,n}|\le c_i\,n^{m_i-1}\rho^{n}$,
and $n^{m_i-1}(\rho/\rho')^{n}\to0$ gives $C>0$ with $|h_n|\le C\rho'^{\,n}$
for all $n$. Hence $\sum_{n\ge0}|h_n|\le C/(1-\rho')=:C_h<\infty$. Using
$\max_k|e_k|\le\|e\|$,
\[
\bigl|(T_\mathbf{h}^\top e)_t\bigr|
=\Bigl|\sum_{k=t}^{H}h_{k-t}e_k\Bigr|
\le\Bigl(\sum_{m\ge0}|h_m|\Bigr)\max_k|e_k|
\le C_h\,\|e\|,
\]
uniformly in  $H$; thus $|(T_\mathbf{h}^\top e)_t|=O(\|e\|)$.

\paragraph{(ii) Case $\rho=1$, unit-modulus roots simple.}
Roots with $|\lambda_i|<1$ contribute vanishing, hence bounded, terms. A
root with $|\lambda_i|=1$ is simple, so $m_i=1$ and $p_i$ is a constant.
Summing the finitely many terms of \eqref{eq:pf_gensol} gives $C>0$ with
$|h_n|\le C$ for all $n$.

\paragraph{(iii) Case $\rho>1$ (or $\rho=1$ with a repeated unit-modulus
root).}
Split the roots by modulus,
$\mathcal{D}=\{i:|\lambda_i|=\rho\}$, $\mathcal{S}=\{i:|\lambda_i|<\rho\}$,
and write $h_n=h^{\mathrm{dom}}_n+h^{\mathrm{sub}}_n$ accordingly. Let
$m=\max_{i\in\mathcal{D}}m_i$ and
$\mathcal{D}^{\star}=\{i\in\mathcal{D}:m_i=m\}$. We show
$\sup_n|h_n|=\infty$ in five steps; the crux (Step~1) is that the specific
initialization $h_0=1$, $h_n=0$ ($n<0$) forces the dominant modes to be
present.

\emph{Step 1 (every mode is present).}
Let $\mathcal{H}(x):=\sum_{n\ge0}h_nx^{n}$. Multiplying the recurrence of
$h_n$ by $x^{n}$ and summing over $n\ge1$, the left side is
$\sum_{n\ge1}h_nx^{n}=\mathcal{H}(x)-h_0=\mathcal{H}(x)-1$, while on the
right, exchanging the finite sum over $j$ with the sum over $n$ and
substituting $\ell=n-j$,
\[
\sum_{n\ge1}h_{n-j}x^{n}
=x^{j}\sum_{\ell\ge1-j}h_{\ell}x^{\ell}
=x^{j}\mathcal{H}(x),
\]
where the last equality holding because $h_{\ell}=0$ for $\ell<0$. Hence
$\mathcal{H}(x)-1=\mathcal{H}(x)\sum_{j=1}^{P}\delta_jx^{j}$, i.e.\
$\mathcal{H}(x)D(x)=1$ with $D(x):=1-\sum_{j=1}^{P}\delta_jx^{j}$, so
$\mathcal{H}(x)=1/D(x)$. The numerator is the constant $1$ exactly because
$h_0=1$ and $h_{\ell}=0$ for $\ell<0$; another initialization would leave
boundary terms $x^{j}\sum_{\ell=1-j}^{-1}h_{\ell}x^{\ell}$ and a nonconstant
numerator. As $D(x)=x^{P}\chi(1/x)$, the roots of $D$ are the $1/\lambda_i$
with multiplicities $m_i$ ($\lambda_i\neq0$ since $\delta_P\neq0$), and
$D(0)=1$ fixes the normalization, so
$\mathcal{H}(x)=\prod_{i}(1-\lambda_i x)^{-m_i}$. In its partial-fraction expansion
$\mathcal{H}(x)=\sum_i\sum_{t=1}^{m_i}A_{i,t}(1-\lambda_ix)^{-t}$, we extract
the top coefficient $A_{i,m_i}$ by multiplying through by
$(1-\lambda_ix)^{m_i}$. On the product form this cancels the $i$-th factor,
\[
(1-\lambda_ix)^{m_i}\mathcal{H}(x)
=\prod_{i'\neq i}(1-\lambda_{i'}x)^{-m_{i'}},
\]
while on the expansion it gives
\[
\begin{aligned}
(1-\lambda_ix)^{m_i}\mathcal{H}(x)
&=\sum_{t=1}^{m_i}A_{i,t}(1-\lambda_ix)^{m_i-t}\\
&+(1-\lambda_ix)^{m_i}\sum_{i'\neq i}\sum_{t}A_{i',t}(1-\lambda_{i'}x)^{-t}.
\end{aligned}
\]
Evaluating at $x=1/\lambda_i$, every term with $m_i-t>0$ vanishes and the
remainder vanishes through its factor $(1-\lambda_ix)^{m_i}$ (its poles
$1/\lambda_{i'}$ being finite there, as $\lambda_{i'}\neq\lambda_i$), so only
$t=m_i$ survives:
\[
A_{i,m_i}=\prod_{i'\neq i}(1-\lambda_{i'}/\lambda_i)^{-m_{i'}}\neq0,
\]
a finite product of nonzero factors.

Finally, using the negative-binomial
expansion
$(1-\lambda_ix)^{-t}=\sum_{n\ge0}\binom{n+t-1}{t-1}\lambda_i^{\,n}x^{n}$,
whose coefficient $\binom{n+t-1}{t-1}$ is a polynomial in $n$ of degree
$t-1$, and matching the coefficient of $x^{n}$ across the two forms of
$\mathcal{H}(x)$ (so that $h_n=\sum_i p_i(n)\lambda_i^{\,n}$) gives
\[
p_i(n)=\sum_{t=1}^{m_i}A_{i,t}\binom{n+t-1}{t-1},
\]
of degree exactly $m_i-1$: the $t=m_i$ term alone attains that degree, with
leading coefficient $a_i=A_{i,m_i}/(m_i-1)!\neq0$, and the lower-$t$ terms
(degree $<m_i-1$) cannot cancel it. Thus
$d:=\max_{i\in\mathcal{D}}\deg p_i=m-1\ge0$, and $a_i\neq0$ for every
$i\in\mathcal{D}^{\star}$.

\emph{Step 2 (a lemma on unit-modulus sums).}

\begin{lemma}\label{lem:expsum}
If $\mu_1,\ldots,\mu_s$ are distinct with $|\mu_r|=1$ and
$\alpha_1,\ldots,\alpha_s$ are not all zero, then
$g(n)=\sum_r\alpha_r\mu_r^{\,n}$ satisfies $|g(n)|\ge\epsilon$ on an
infinite set, for some $\epsilon>0$.
\end{lemma}
\begin{proof}
The Ces\`aro average $\frac1N\sum_{n\le N}|g(n)|^{2}$ tends to
$\sum_r|\alpha_r|^{2}>0$: diagonal terms give $|\alpha_r|^2$, while for
$r\neq r'$, $\omega=\mu_r\overline{\mu_{r'}}$ has $|\omega|=1$,
$\omega\neq1$, so $|\frac1N\sum_{n\le N}\omega^{n}|\le2/(N|1-\omega|)\to0$.
If $\limsup_n|g(n)|=0$ then $|g(n)|\to0$ and the average vanishes, a
contradiction; the claim follows.
\end{proof}
(The bound holds only along an infinite set: $g$ may vanish at individual
$n$, e.g.\ $g(n)=2\cos\theta n$.)

\emph{Step 3 (subdominant part negligible).}
If $\mathcal{S}=\emptyset$ this is trivial. Else let
$\rho_{\mathrm{sub}}=\max_{i\in\mathcal{S}}|\lambda_i|<\rho$. As every
$p_i$ has degree $\le P-1$, $|h^{\mathrm{sub}}_n|\le
Bn^{P-1}\rho_{\mathrm{sub}}^{n}$, so with $n^{-d}\le1$ ($n\ge1$),
\begin{equation}
\frac{|h^{\mathrm{sub}}_n|}{n^{d}\rho^{n}}
\le B\,n^{P-1}\bigl(\rho_{\mathrm{sub}}/\rho\bigr)^{n}\to0,
\label{eq:pf_subneg}
\end{equation}
since $n^{\beta}q^{n}\to0$ for any $q\in(0,1)$, $\beta\ge0$.

\emph{Step 4 (dominant rate).}
With $\mu_i=\lambda_i/\rho$ (distinct, unimodular),
\[
\frac{h^{\mathrm{dom}}_n}{n^{d}\rho^{n}}=g(n)+\eta(n),
\qquad
g(n)=\sum_{i\in\mathcal{D}^{\star}}a_i\mu_i^{\,n},
\]
where $\eta$ collects the degree-$\le d{-}1$ remainders of the
$\mathcal{D}^{\star}$ modes and all of
$\mathcal{D}\setminus\mathcal{D}^{\star}$ (so $\eta\equiv0$ if $d=0$). Each
such term is a polynomial of degree $\le d-1$ times a unimodular
$\mu_i^{n}$, hence $O(n^{-1})$ after dividing by $n^{d}$; so
$\eta(n)\to0$. By Step~1 the $a_i$ ($i\in\mathcal{D}^{\star}$) are nonzero,
so Lemma~\ref{lem:expsum} applies to $g$.

\emph{Step 5 (conclusion).}
Take $\epsilon>0$ and infinite $\mathcal{N}_0$ with $|g|\ge\epsilon$ there.
Choose $N_1$ with $|\eta(n)|\le\epsilon/4$ ($n\ge N_1$), giving
$|h^{\mathrm{dom}}_n|/(n^{d}\rho^{n})\ge3\epsilon/4$ on $\mathcal{N}_0$
beyond $N_1$; and $N_2$ with
$|h^{\mathrm{sub}}_n|/(n^{d}\rho^{n})\le\epsilon/4$ ($n\ge N_2$) by
\eqref{eq:pf_subneg}. On $\mathcal{N}=\{n\in\mathcal{N}_0:n\ge\max(N_1,N_2)\}$,
\[
\frac{|h_n|}{n^{d}\rho^{n}}
\ge\frac{3\epsilon}{4}-\frac{\epsilon}{4}=\frac{\epsilon}{2},
\]
so $|h_n|\ge\frac\epsilon2 n^{d}\rho^{n}$ on $\mathcal{N}$. Since
$\rho\ge1$ and $d\ge1$ whenever $\rho=1$ (a repeated unit root has
$m\ge2$), the bound diverges along $\mathcal{N}$; hence
$\sup_n|h_n|=\infty$. \hfill$\qed$
\subsection{Proof of Proposition~\ref{prop:sensitivity}}

\paragraph{Convolution identity.}
Since $h_0=1$ does not depend on $\boldsymbol{\delta}$, we have
$g_0^{(j)}=0$, and differentiating $h_n=\sum_{k=1}^{P}\delta_k h_{n-k}$
with respect to $\delta_j$ gives, for $n\ge1$,
\[
g_n^{(j)}=h_{n-j}+\sum_{k=1}^{P}\delta_k\,g_{n-k}^{(j)},
\qquad g_n^{(j)}=0\ \ (n<0),
\]
so $g^{(j)}$ obeys the same recurrence as $\mathbf{h}$, forced by
$h_{n-j}$. Let $\mathcal{G}^{(j)}(x):=\sum_{n\ge0}g_n^{(j)}x^{n}$.
Multiplying by $x^{n}$ and summing over $n\ge1$, the left side is
$\mathcal{G}^{(j)}(x)-g_0^{(j)}=\mathcal{G}^{(j)}(x)$, while the shift
identities $\sum_{n\ge1}h_{n-j}x^{n}=x^{j}\mathcal{H}(x)$ and
$\sum_{n\ge1}g^{(j)}_{n-k}x^{n}=x^{k}\mathcal{G}^{(j)}(x)$ (both by the
substitution used in Step~1, the negative-index terms vanishing) give
\[
\begin{aligned}
\mathcal{G}^{(j)}(x)
&=x^{j}\mathcal{H}(x)+\mathcal{G}^{(j)}(x)\sum_{k=1}^{P}\delta_kx^{k},\\
\text{i.e.}\quad
\mathcal{G}^{(j)}(x)D(x)&=x^{j}\mathcal{H}(x).
\end{aligned}
\]
Dividing by $D$ and substituting
$1/D=\mathcal{H}$,
\begin{equation}
\mathcal{G}^{(j)}(x)=\frac{x^{j}\mathcal{H}(x)}{D(x)}
=x^{j}\mathcal{H}(x)^{2},
\label{eq:pf_G}
\end{equation}
i.e.\ $g_n^{(j)}=(\mathbf{h}*\mathbf{h}^{(j)})_n$ with
$h_n^{(j)}=h_{n-j}$, since $\mathcal{H}^{2}$ generates $\mathbf{h}*\mathbf{h}$
and multiplication by $x^{j}$ shifts by $j$. Each item now follows from the
corresponding bound on $\mathbf{h}$, using
\begin{equation}
|g_n^{(j)}|\le\sum_{i=0}^{n}|h_i|\,|h_{n-i-j}| .
\label{eq:pf_gconv}
\end{equation}
Throughout we use
$|\phi_{ev}(n)|=\bigl|\sum_{k>n}e_kv_{k-n}\bigr|\le\|e\|\,\|\mathbf{v}\|$
(Cauchy--Schwarz).

\paragraph{(i) Case $\rho<1$.}
Fix $\rho'\in(\rho,1)$ and pick $\rho''\in(\rho,\rho')$. By
Proposition~\ref{prop:stability}(1), $|h_n|\le C\rho''^{\,n}$ for some
$C>0$, so \eqref{eq:pf_gconv} gives
$|g_n^{(j)}|\le C^{2}(n+1)\rho''^{\,n-j}\le C_g\rho'^{\,n}$, since
$(n+1)(\rho''/\rho')^{n}\to0$. Hence
$\sum_{n\ge0}|g_n^{(j)}|\le C_g/(1-\rho')=:C_g'<\infty$ and
\[
|\mathrm{Term~I}|
=\Bigl|\sum_{n=0}^{H-1}g_n^{(j)}\phi_{ev}(n)\Bigr|
\le C_g'\,\|e\|\,\|\mathbf{v}\|,
\]
independently of $H$.

\paragraph{(ii) Case $\rho=1$, unit-modulus roots simple.}
By Proposition~\ref{prop:stability}(ii), $|h_n|\le C$, so
\eqref{eq:pf_gconv} gives $|g_n^{(j)}|\le C^{2}(n+1)=O(n)$ and
\[
|\mathrm{Term~I}|
\le C^{2}\sum_{n=0}^{H-1}(n+1)\,\|e\|\,\|\mathbf{v}\|
=O\bigl(H^{2}\|e\|\,\|\mathbf{v}\|\bigr).
\]

\paragraph{(iii) Case $\rho>1$, or $\rho=1$ with a repeated unit-modulus
root.}
By \eqref{eq:pf_G},
$\mathcal{G}^{(j)}(x)=x^{j}\prod_{i=1}^{r}(1-\lambda_ix)^{-2m_i}$: the
poles are those of $\mathcal{H}$, at $1/\lambda_i$, with multiplicities
doubled to $2m_i$. The numerator $x^{j}$ vanishes only at $x=0$, and
$\lambda_i\neq0$, so no pole is cancelled. The argument of
Proposition~\ref{prop:stability}(3) therefore applies verbatim to
$g^{(j)}$ with $m_i$ replaced by $2m_i$: every mode is present with nonzero
dominant leading coefficients, and Lemma~\ref{lem:expsum} yields $c>0$ and
an infinite set $\mathcal{N}$ with
$|g_n^{(j)}|\ge c\,n^{2m-1}\rho^{n}$ on $\mathcal{N}$, where $m$ is the
maximal multiplicity of a root of modulus $\rho$. This bound diverges along
$\mathcal{N}$ in both sub-cases: if $\rho>1$ the exponential factor
diverges, while if $\rho=1$ a repeated unit root gives $m\ge2$, hence
$2m-1\ge3$ and $c\,n^{2m-1}\to\infty$. Therefore
$\sup_n|g_n^{(j)}|=\infty$.

\subsection{Proof of Theorem~\ref{thm:l1}}

Let $\lambda$ be any root of the characteristic polynomial
$\chi(\lambda)=\lambda^{P}-\sum_{j=1}^{P}\delta_j\lambda^{P-j}$, so that
$
\lambda^{P}=\sum_{j=1}^{P}\delta_j\lambda^{P-j}.
\label{eq:pf_l1_char}$
Suppose, for contradiction, that $|\lambda|\ge1$. Taking moduli 
and applying the triangle inequality,
\[
|\lambda|^{P}
\;\le\;\sum_{j=1}^{P}|\delta_j|\,|\lambda|^{P-j}
\;\le\;|\lambda|^{P-1}\sum_{j=1}^{P}|\delta_j|
\;=\;|\lambda|^{P-1}\,\|\boldsymbol{\delta}\|_1 ,
\]
where the second inequality uses $|\lambda|^{P-j}\le|\lambda|^{P-1}$ for
$j\ge1$, valid since $|\lambda|\ge1$. As $|\lambda|\ge1$ implies
$|\lambda|^{P-1}>0$, dividing by $|\lambda|^{P-1}$ gives
$|\lambda|\le\|\boldsymbol{\delta}\|_1<1$, contradicting $|\lambda|\ge1$.

Hence every root satisfies $|\lambda|<1$, and since $\rho$ is the largest
root modulus, $\rho<1$. \hfill$\qed$

\subsection{Proof of Proposition~\ref{prop:autocov}}

\paragraph{Setting.}
We work at the population level under the data-generating process, taking
$\mathbf{X}:=(x_{1-P},\ldots,x_{L+H})^{\top}$, so every lagged value
required by the differencing operator and by the shifts
$\mathbf{x}^{(j)}$ is an observation of $(x_t)$ and
$\mathbb{E}[x_ux_v]=\gamma(u-v)$ holds throughout.

\paragraph{Step 1: everything is affine in $\mathbf{X}$.}
With a linear backbone, $\mathbf{z}=A\mathbf{X}$ (from \eqref{eq:diff}),
$\mathbf{c}=C\mathbf{X}$ and $\mathbf{x}_{\mathrm{f}}=S\mathbf{X}$ for
matrices determined by $\boldsymbol{\delta}$. By \eqref{eq:backbone} and
Proposition~\ref{prop:convolution},
$\hat{\mathbf{x}}_{\mathrm{f}}=T_{\mathbf{h}}(W_bA\mathbf{X}+\mathbf{b}+C\mathbf{X})$,
so $e=M\mathbf{X}+\mathbf{b}'$ with $M:=T_{\mathbf{h}}(W_bA+C)-S$ and
$\mathbf{b}':=T_{\mathbf{h}}\mathbf{b}$, both independent of $\mathbf{X}$.
Hence
\begin{equation}
\mathbf{p}^{\top}=(T_{\mathbf{h}}^{\top}e)^{\top}W_b
=\mathbf{X}^{\top}N+\mathbf{q}^{\top},
\qquad
N:=M^{\top}T_{\mathbf{h}}W_b,
\label{eq:pf_p_linear}
\end{equation}
whose entries $N_{v,s}$ are indexed by the coordinates
$v\in\{1-P,\ldots,L+H\}$ of $\mathbf{X}$ and $s\in\{1,\ldots,L-1\}$ of
$\mathbf{z}$.

\paragraph{Step 2: reduction to second moments.}
Since $(x^{(j)}_{\scriptscriptstyle 2:L})_s=x_{s+1-j}$, and
$s+1-j\ge 2-P$ lies in the range of $\mathbf{X}$, we may write
$\mathbf{x}^{(j)}_{\scriptscriptstyle 2:L}=E_j\mathbf{X}$ with
$(E_j)_{s,u}=\mathbf{1}\{u=s+1-j\}$. By \eqref{eq:pf_p_linear},
\[
\mathbf{p}^{\top}\mathbf{x}^{(j)}_{\scriptscriptstyle 2:L}
=\mathbf{X}^{\top}NE_j\mathbf{X}+\mathbf{q}^{\top}E_j\mathbf{X}.
\]
Taking expectations, the linear term vanishes as $\mathbb{E}[\mathbf{X}]=0$,
and $\mathbb{E}[\mathbf{X}^{\top}B\mathbf{X}]=\operatorname{tr}(B\Gamma)$
for deterministic $B$ gives
$\mathbb{E}[\mathbf{p}^{\top}\mathbf{x}^{(j)}_{\scriptscriptstyle 2:L}]
=\operatorname{tr}(NE_j\Gamma)$, with
$\Gamma_{u,v}=\mathbb{E}[x_ux_v]=\gamma(u-v)$.

\paragraph{Step 3: extracting the lag.}
Expanding the trace entrywise and using $(E_j)_{s,u}=\mathbf{1}\{u=s+1-j\}$
to collapse the sum over $u$:
\begin{align*}
\operatorname{tr}(NE_j\Gamma)
&=\sum_{s=1}^{L-1}\ \sum_{v=1-P}^{L+H}N_{v,s}\,\Gamma_{s+1-j,\,v}\\
&=\sum_{s,v}N_{v,s}\,\gamma(s+1-j-v).
\end{align*}
Setting $m:=v-s-1$, so that $s+1-j-v=-(m+j)$, and using
$\gamma(-k)=\gamma(k)$,
\begin{equation}
\mathbb{E}\bigl[\mathbf{p}^{\top}\mathbf{x}^{(j)}_{\scriptscriptstyle 2:L}\bigr]
=\sum_{m=1-P-L}^{L+H-2}\Phi(m)\,\gamma(m+j),
\label{eq:pf_phi}
\end{equation}
where $\Phi(m):=\sum_{v-s-1=m}N_{v,s}$, the endpoints of the range being
attained at $(v,s)=(1-P,L-1)$ and $(L+H,1)$. The weights $\Phi(m)$ are
sums of entries of $N$, hence depend on
$(\boldsymbol{\delta},W_b,\mathbf{b})$ but not on $j$: the lag enters only
through the argument of $\gamma$.

\section{More details of AdaRDiff}
\label{app:pseudocode}
\subsection{Implementation Details}

\paragraph{Overview.}
AdaRDiff wraps a forecasting backbone in two operations that share the same
learnable weights $\boldsymbol{\delta}\in\mathbb{R}^{C\times P}$: a causal
differencing pass on the input, and a reconstruction pass on the backbone
output (Algorithm~\ref{alg:forward}). Both are implemented as depthwise \texttt{conv1d} over the channel
dimension, so the module adds $C\!\times\!P$ parameters and no sequential
work beyond the $O(\log H)$ recurrence of Algorithm~\ref{alg:h}. Training
proceeds in two phases (Section~\ref{sec:training}): Phase~1 supervises the
backbone in residual space, bypassing reconstruction entirely; Phase~2
optimizes the forecast loss through the full pipeline.

\begin{algorithm}[t]
\caption{AdaRDiff forward pass}
\label{alg:forward}
\begin{algorithmic}[1]
\REQUIRE context $\mathbf{x}\in\mathbb{R}^{C\times L}$, phase
$\in\{1,2\}$; target $\mathbf{x}_{\mathrm{f}}\in\mathbb{R}^{C\times H}$
\STATE initialize differencing weights $\boldsymbol{\delta}$, gain
$\alpha\leftarrow1$, and backbone parameters $\theta$
\STATE \textbf{if} reparam \textbf{then}
$\boldsymbol{\delta}\leftarrow\alpha\,\boldsymbol{\delta}/
\|\boldsymbol{\delta}\|_1$ \COMMENT{per channel, Eq.~\eqref{eq:reparam}}
\STATE \textbf{if} RevIN \textbf{then}
$\mathbf{x}\leftarrow(\mathbf{x}-\mu)/\sigma$, with $\mu,\sigma$ the
per-channel mean and standard deviation of $\mathbf{x}$
\STATE $\mathbf{z}\leftarrow\mathbf{x}[:,1{:}]-
\texttt{conv1d}(\mathbf{x},\boldsymbol{\delta})$
\COMMENT{differencing, Eq.~\eqref{eq:diff}}
\STATE $\hat{\mathbf{z}}\leftarrow f_\theta(\mathbf{z})$
\COMMENT{predicted residuals}
\IF{phase $=1$}
  \STATE \textbf{if} RevIN \textbf{then}
  $\mathbf{x}_{\mathrm{f}}\leftarrow(\mathbf{x}_{\mathrm{f}}-\mu)/\sigma$
  \COMMENT{same statistics as $\mathbf{x}$}
  \STATE $\mathbf{z}_{\mathrm{f}}\leftarrow\mathbf{x}_{\mathrm{f}}-
  \texttt{conv1d}([\,\mathbf{x}[:,-P{:}],\,\mathbf{x}_{\mathrm{f}}\,],
  \boldsymbol{\delta})$
  \COMMENT{difference the target}
  \RETURN $(\hat{\mathbf{z}},\mathbf{z}_{\mathrm{f}})$
  \COMMENT{loss $\|\hat{\mathbf{z}}-\mathbf{z}_{\mathrm{f}}\|^{2}$}
\ELSE
  \STATE $\mathbf{h}\leftarrow\textsc{Sequence-}\mathbf{h}
  (\boldsymbol{\delta},H)$ \COMMENT{Alg.~\ref{alg:h}}
  \STATE $\mathbf{c}\leftarrow\texttt{conv1d}(\mathbf{x}[:,-P{:}],
  \boldsymbol{\delta})$ \COMMENT{initial conditions}
  \STATE $\hat{\mathbf{x}}_{\mathrm{f}}\leftarrow
  \texttt{conv1d}(\hat{\mathbf{z}}+\mathbf{c},\mathbf{h})$
  \COMMENT{reconstruction, Prop.~\ref{prop:convolution}}
  \STATE \textbf{if} RevIN \textbf{then} $\hat{\mathbf{x}}_{\mathrm{f}}
  \leftarrow\sigma\hat{\mathbf{x}}_{\mathrm{f}}+\mu$
  \RETURN $\hat{\mathbf{x}}_{\mathrm{f}}$
  \COMMENT{loss
  $\|\hat{\mathbf{x}}_{\mathrm{f}}-\mathbf{x}_{\mathrm{f}}\|^{2}$}
\ENDIF
\end{algorithmic}
\end{algorithm}

\paragraph{Reconstruction components.}
Reconstruction needs three ingredients:

\emph{The sequence $\mathbf{h}$.} Stacking $P$ consecutive terms into the
state $s_n=(h_n,\ldots,h_{n-P+1})^{\top}$ turns the recurrence
\eqref{eq:recurrence} into a one-step linear map $s_{n+1}=Fs_n$, where $F$
is the companion matrix of $\boldsymbol{\delta}$; hence $s_n=F^{n}s_0$ with
$s_0=(1,0,\ldots,0)^{\top}$, and $h_n=(s_n)_1$. Computing $s_0,\ldots,s_{H-1}$
one at a time costs $H$ sequential steps. Repeated squaring reduces this to
$\lceil\log_2 H\rceil$: the key identity is that $F^{\ell}$ advances any
state by $\ell$ positions, $F^{\ell}s_n=s_{n+\ell}$. So if we have already
computed the first $\ell$ states $s_0,\ldots,s_{\ell-1}$ and the matrix
$G=F^{\ell}$, applying $G$ to all of them at once yields the next $\ell$
states $s_\ell,\ldots,s_{2\ell-1}$; squaring $G$ then gives $F^{2\ell}$ for
the following round. Each round doubles the number of known states, so
$\lceil\log_2 H\rceil$ rounds cover the whole horizon, and within a round
the single matrix product $G\,[s_0\cdots s_{\ell-1}]$ produces all new
states in parallel (Algorithm~\ref{alg:h}).

\emph{The initial conditions $\mathbf{c}$.} The terms
$c_k=\sum_{j\ge k}\delta_jx_{L+k-j}$ are a single convolution of the last
$P$ context values with $\boldsymbol{\delta}$, zero-padded so that taps
reaching past the context contribute nothing.

\emph{The forecast.} With $\mathbf{v}=\hat{\mathbf{z}}+\mathbf{c}$, the
reconstruction $\hat{x}_{L+k}=\sum_{m<k}h_m v_{k-m}$ is the causal
convolution of $\mathbf{v}$ with $\mathbf{h}$, i.e.\ the Toeplitz product
$T_{\mathbf{h}}\mathbf{v}$.

\begin{algorithm}[ht]
\caption{$\textsc{Sequence-}\mathbf{h}(\boldsymbol{\delta},H)$ by repeated
squaring}
\label{alg:h}
\begin{algorithmic}[1]
\STATE $F\leftarrow$ companion matrix of $\boldsymbol{\delta}$
\STATE $S\leftarrow[\,s_0\,]$ with $s_0=(1,0,\ldots,0)^{\top}$;\quad
$G\leftarrow F$ \COMMENT{$S$ holds $s_0,\ldots,s_{|S|-1}$; $G=F^{|S|}$}
\WHILE{$|S|<H$}
  \STATE $S\leftarrow[\,S,\;G\,S\,]$
  \COMMENT{$G\,s_n=s_{n+|S|}$: append the next $|S|$ states}
  \STATE $G\leftarrow G^{2}$
  \COMMENT{$F^{\ell}\!\to\!F^{2\ell}$ for the next round}
\ENDWHILE
\RETURN $h_n=(s_n)_1$ for $n=0,\ldots,H-1$
\end{algorithmic}
\end{algorithm}

\paragraph{Options.}
Two components are optional hyperparameters. RevIN normalizes each channel
before differencing and inverts the transform on the forecast, handling
distribution shift. The reparameterization of line~2 constrains
$\|\boldsymbol{\delta}\|_1=|\alpha|$ with $\alpha$ learned per channel.

\section{Supplementary Details of the Experiments}
\label{app:suppl-details}
\subsection{Experiment Settings}
\paragraph{Datasets.}
We evaluate on widely adopted LTSF benchmarks: the ETT series (ETTh1, ETTh2,
ETTm1, ETTm2)~\cite{informer}, Weather, Traffic, Electricity, and
Solar-Energy~\cite{10.1145/3209978.3210006}. Dataset splitting and
normalization follow prior work~\cite{liu2024itransformer,lin2024cyclenet}:
the ETT datasets use a $6\!:\!2\!:\!2$ train/validation/test split, and the
remaining datasets a $7\!:\!1\!:\!2$ split. Table~\ref{tab:datasets}
summarizes each dataset. All exhibit periodic structure---daily, and
additionally weekly for Electricity and Traffic. Combined with the sampling
frequency, this determines the dominant period of each dataset, e.g.\ $24$
steps for the hourly ETTh variants and $168$ for Electricity. These periods
can be inferred from an autocorrelation (ACF) analysis, following the
procedure of CycleNet~\cite{lin2024cyclenet}, and we use them to guide the
choice of window size $P$.
\paragraph{Hyperparameter search.}
All hyperparameters are selected per dataset on the validation set, by grid
search over the following ranges. The window size $P$ is tuned as described
in the main text ($\{4,24,96\}$ on the ETT datasets, the dominant period
elsewhere). The initialization of the differencing weights
$\boldsymbol{\delta}$ is chosen among zero
($\boldsymbol{\delta}=\mathbf{0}$), uniform ($\delta_j=1/P$), and
first-order ($\delta_1=1$, else $0$); RevIN~\cite{revin} and the
$\ell_1$-reparameterization of Eq.~\eqref{eq:reparam} are each toggled on or
off. Optimization hyperparameters are swept over initial learning rates
$\{1\!\times\!10^{-3}, 2\!\times\!10^{-3}, 5\!\times\!10^{-3},
1\!\times\!10^{-2}\}$ and batch sizes $\{32, 64, 256\}$. For the MLP
backbone, the hidden dimension is additionally selected from
$\{128, 512, 1024\}$.

\begin{table}[h]
\centering
\setlength{\tabcolsep}{5pt}
\renewcommand{\arraystretch}{1.2}
\scalebox{0.72}{
\begin{tabular}{l|cccccc}
\toprule
\textbf{Dataset} & \makecell{ETTh1\\ETTh2} & \makecell{ETTm1\\ETTm2} & Electricity & \makecell{Solar-\\Energy} & Traffic & Weather \\
\midrule
Timesteps   & 17{,}420 & 69{,}680 & 26{,}304 & 52{,}560 & 17{,}544 & 52{,}696 \\
Channels    & 7        & 7        & 321      & 137      & 862      & 21 \\
Frequency   & 1\,h     & 15\,min  & 1\,h     & 10\,min  & 1\,h     & 10\,min \\
Periodicity & Daily    & Daily    & \makecell{Daily,\\Weekly} & Daily & \makecell{Daily,\\Weekly} & Daily \\
\makecell[l]{Dominant\\Period} & 24 & 96 & 168 & 144 & 168 & 144 \\
\bottomrule
\end{tabular}}
\caption{Dataset information. The dominant period is given in timesteps;
for datasets with both daily and weekly cycles it is the longer (weekly)
period.}
\label{tab:datasets}
\end{table}
 
\subsection{More experimental results}

\paragraph{Comparison under the unified TFB evaluation.}
Table~\ref{tab:multivariate_main} reports a complementary comparison under
the unified evaluation protocol of the Time Series Forecasting Benchmark
(TFB)~\cite{qiu2024tfb}, which standardizes datasets, splits, and metrics
across methods, with each method is given its best result over
look-back lengths $\{96,336,512\}$ rather than a fixed look-back window. This
setting favors strong baselines, since each is tuned to its most
advantageous context length; we nonetheless evaluate AdaRDiff over the same
lengths and report its best result for a fair comparison. Even so, AdaRDiff
attains the best accuracy on the large majority of dataset--horizon pairs:
its linear and MLP backbones together take the top MSE on $29$ of $35$
settings and the top MAE on $22$, dominating the first-place count over
every competing method, including the recent PhaseFormer~\cite{niu2026phaseformer}, MixLinear~\cite{ma2026mixlinear}, and
Amplifier~\cite{10.1609/aaai.v39i11.33267}. On dataset-level average MSE, AdaRDiff improves over the strongest
baseline on six of the seven datasets, with the largest margin on ETTh2
($0.331$ against $0.347$ for MixLinear), followed by Weather ($0.215$ against
$0.222$) and ETTm2 ($0.249$ against $0.256$). The sole exception is Traffic,
where PatchTST edges ahead ($0.397$ against $0.399$). These results show that
AdaRDiff performs consistently across look-back lengths and evaluation
setups, remaining state-of-the-art with only a lightweight linear or MLP
backbone.
\begin{figure*}[t]
    \centering
    \includegraphics[width=0.9\textwidth]{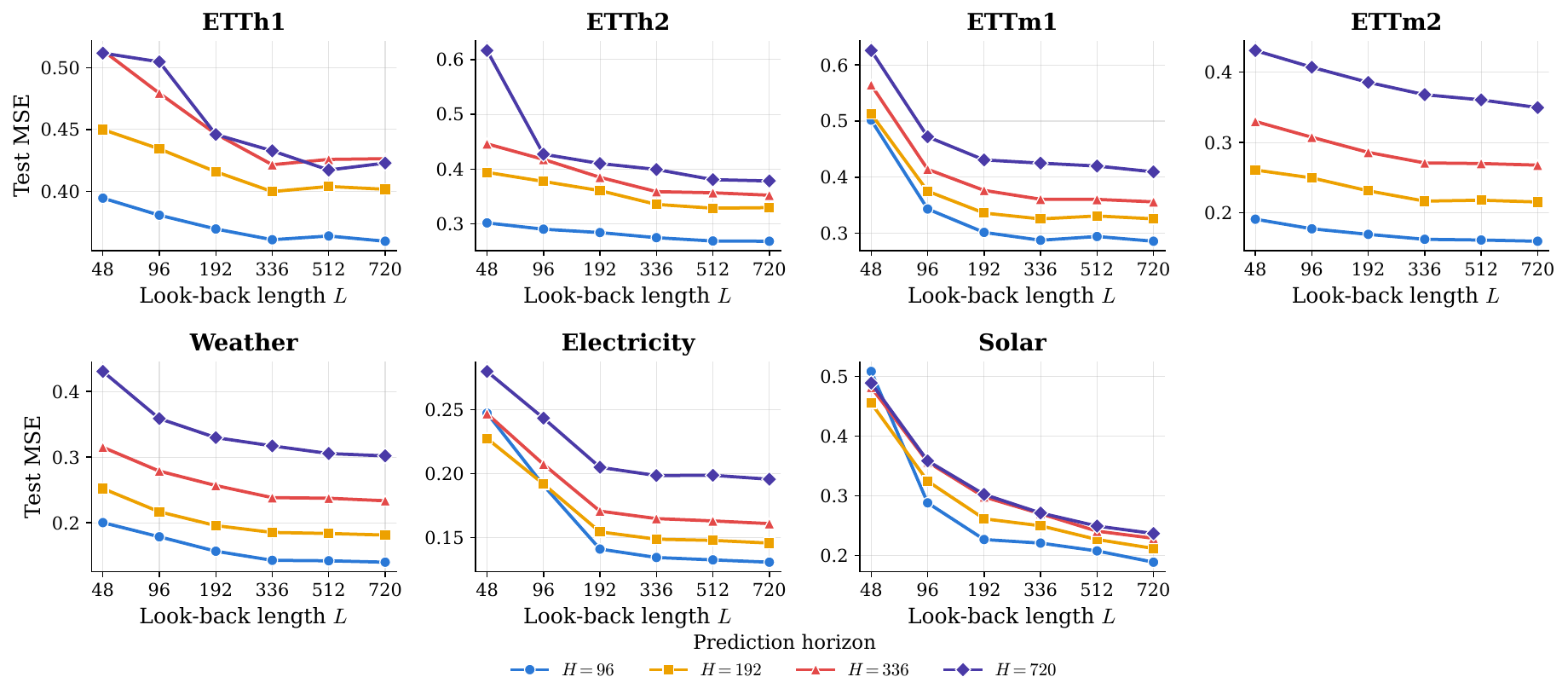}
    \caption{Effect of the look-back length $L$ on AdaRDiff-Linear's test MSE,
    for each dataset and forecast horizon $H\in\{96,192,336,720\}$.}
    \label{fig:contextlen_effect}
\end{figure*}
\paragraph{Effect of the look-back length.}
Figure~\ref{fig:contextlen_effect} reports AdaRDiff-Linear's test MSE as the
look-back length $L$ varies over $\{48,96,192,336,512,720\}$ at each forecast
horizon. Several consistent trends emerge. First, accuracy improves almost
monotonically with $L$: a longer history gives the differencing operator more
context to identify the relevant lags, and the shortest window ($L=48$) is
the weakest on every dataset. Second, these gains saturate quickly at short
horizons: beyond $L\!\approx\!336$ the $H=96$ curves are nearly flat, and on
ETTh1 the longer $L=512$ window is no better than $L=336$. Third, and more
notably, the benefit of a longer look-back grows with the forecast horizon:
extending $L$ from $48$ to $720$ reduces MSE far more at $H=720$ than at
$H=96$ (e.g.\ a drop of $0.089$ versus $0.035$ on ETTh1, and $0.081$ versus
$0.031$ on ETTm2), since forecasting further ahead requires more history to
anchor the trend and seasonal structure. Datasets with longer periodic structure, such as
Electricity and Solar-Energy, continue to benefit from windows up to $L=720$,
consistent with their weekly and daily cycles. Overall, AdaRDiff extracts
most of its short-horizon accuracy from a moderate look-back while making
effective use of longer histories precisely where they matter most---at long
horizons.
\paragraph{Qualitative forecasts.}
Figure~\ref{fig:forecast_panels} shows AdaRDiff-Linear's forecasts (horizon $H=96$) against the ground truth on representative channels
from six datasets. The forecasts follow the dominant temporal structure of
each series: the sharp recurring spikes of Traffic, the daily cycles of
Electricity and the ETT variants, and the day--night plateaus of
Solar-Energy, where AdaRDiff correctly holds the low nighttime level and
rises on the daytime ramp. Errors appear mainly at the peaks of the most
irregular channels (e.g.\ Traffic channel~0), where the amplitude of
individual spikes is occasionally underestimated, but the phase and period
of the pattern are consistently recovered. These examples illustrate how the
learned differencing operator reconstructs both periodic and trending
behavior directly in the original series space.
\begin{figure*}[t]
    \centering
    \includegraphics[width=\textwidth]{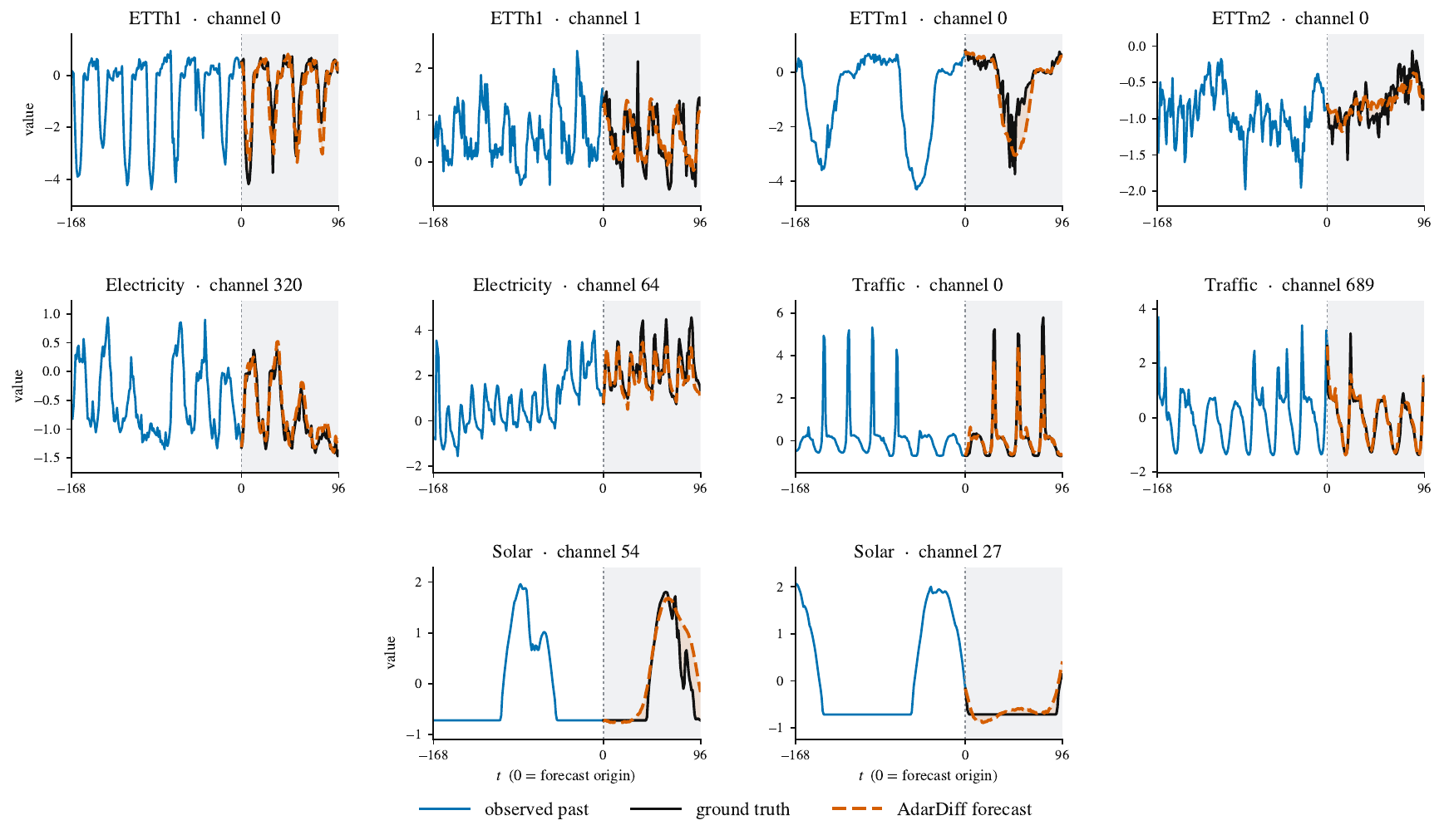}
    \caption{Qualitative forecasts of AdaRDiff-Linear (horizon $H=96$) on representative channels from six datasets. Each panel
    shows the observed past (blue), the ground-truth future (black), and the
    AdaRDiff forecast (orange dashed); the shaded region marks the forecast
    horizon.}
    \label{fig:forecast_panels}
\end{figure*}

\definecolor{OursColor}{HTML}{8B0000}
\begin{table*}[ht]
    \centering
    {\fontsize{10pt}{10pt}\selectfont
    \scalebox{0.52}{
    \begin{tabular}{l | c |c c|c c|c c|c c|c c|c c|c c|c c|c c|c c|c c|c c}
    \toprule[2pt]
       \multicolumn{2}{c|}{Models} & \multicolumn{2}{c|}{\makecell{\textcolor{OursColor}{\textbf{AdaRDiff-Linear}}\\(2026)}} & \multicolumn{2}{c|}{\makecell{\textcolor{OursColor}{\textbf{AdaRDiff-MLP}}\\(2026)}}
       & \multicolumn{2}{c|}{\makecell{PhaseFormer\\(\citeyear{niu2026phaseformer})}} 
       & \multicolumn{2}{c|}{\makecell{MixLinear\\(\citeyear{ma2026mixlinear})}}
       & \multicolumn{2}{c|}{\makecell{Amplifier\\(\citeyear{10.1609/aaai.v39i11.33267})}} 
       & \multicolumn{2}{c|}{\makecell{iTransformer\\(\citeyear{liu2024itransformer})}}
       & \multicolumn{2}{c|}{\makecell{TimeMixer\\(\citeyear{wang2023timemixer})}}
       &
       \multicolumn{2}{c|}{\makecell{Crossformer\\(\citeyear{zhang2023crossformer})}}
       & \multicolumn{2}{c|}{\makecell{DLinear\\(\citeyear{dlinear})}} & \multicolumn{2}{c|}{\makecell{PatchTST\\(\citeyear{patchtst})}} & \multicolumn{2}{c|}{\makecell{TimesNet\\(\citeyear{wu2023timesnet})}} & \multicolumn{2}{c}{\makecell{stationary\\(\citeyear{liu2022nonstationary})}} \\
       \cmidrule(lr){1-2}\cmidrule(lr){3-26}
       \multicolumn{2}{c|}{Metrics} & MSE & MAE & MSE & MAE & MSE & MAE & MSE & MAE & MSE & MAE & MSE & MAE & MSE & MAE & MSE & MAE & MSE & MAE & MSE & MAE & MSE & MAE & MSE & MAE \\
       \midrule
         & 96   & \textcolor{red}{\textbf{0.355}} & \textcolor{red}{\textbf{0.386}} & \underline{\textcolor{blue}{0.360}} & 0.391 & 0.386 & 0.397 & 0.365 & \underline{\textcolor{blue}{0.387}} & 0.376 & 0.393 & 0.386 & 0.405 & 0.372 & 0.401 & 0.411 & 0.435 & 0.379 & 0.403 & 0.377 & 0.397 & 0.389 & 0.412 & 0.591 & 0.524   \\
         & 192   & \textcolor{red}{\textbf{0.393}} & \underline{\textcolor{blue}{0.411}} & 0.409 & 0.423 & 0.412 & 0.420 & \underline{\textcolor{blue}{0.397}} & \textcolor{red}{\textbf{0.406}} & 0.414 & 0.420 & 0.424 & 0.440 & 0.413 & 0.430 & 0.409 & 0.438 & 0.408 & 0.419 & 0.409 & 0.425 & 0.440 & 0.443 & 0.615 & 0.540   \\
         & 336   & \underline{\textcolor{blue}{0.421}} & \underline{\textcolor{blue}{0.426}} & 0.428 & 0.435 & 0.442 & 0.431 & \textcolor{red}{\textbf{0.418}} & \textcolor{red}{\textbf{0.421}} & 0.442 & 0.446 & 0.449 & 0.460 & 0.438 & 0.450 & 0.433 & 0.457 & 0.440 & 0.440 & 0.431 & 0.444 & 0.523 & 0.487 & 0.632 & 0.551   \\
         & 720   & 0.417 & 0.444 & 0.454 & 0.460 & \textcolor{red}{\textbf{0.414}} & \underline{\textcolor{blue}{0.440}} & \underline{\textcolor{blue}{0.416}} & \textcolor{red}{\textbf{0.437}} & 0.480 & 0.479 & 0.495 & 0.487 & 0.486 & 0.484 & 0.501 & 0.514 & 0.471 & 0.493 & 0.457 & 0.477 & 0.521 & 0.495 & 0.828 & 0.658   \\
         \cmidrule(lr){2-26}
       \multirow{-5}*{\rotatebox{90}{ETTh1}}  & Avg   & \textcolor{red}{\textbf{0.397}} & \underline{\textcolor{blue}{0.417}} & 0.413 & 0.427 & 0.413 & 0.422 & \underline{\textcolor{blue}{0.399}} & \textcolor{red}{\textbf{0.413}} & 0.428 & 0.434 & 0.439 & 0.448 & 0.427 & 0.441 & 0.439 & 0.461 & 0.424 & 0.439 & 0.419 & 0.436 & 0.468 & 0.459 & 0.666 & 0.568   \\
         \midrule
         & 96   & \textcolor{red}{\textbf{0.267}} & \textcolor{red}{\textbf{0.333}} & \underline{\textcolor{blue}{0.268}} & \underline{\textcolor{blue}{0.337}} & 0.284 & 0.350 & 0.291 & 0.343 & 0.291 & 0.342 & 0.297 & 0.348 & 0.281 & 0.351 & 0.728 & 0.603 & 0.300 & 0.364 & 0.274 & \underline{\textcolor{blue}{0.337}} & 0.334 & 0.370 & 0.347 & 0.387   \\
         & 192   & \textcolor{red}{\textbf{0.325}} & \textcolor{red}{\textbf{0.373}} & \underline{\textcolor{blue}{0.329}} & \underline{\textcolor{blue}{0.375}} & 0.347 & 0.387 & 0.344 & 0.380 & 0.355 & 0.400 & 0.372 & 0.403 & 0.349 & 0.387 & 0.723 & 0.607 & 0.387 & 0.423 & 0.348 & 0.384 & 0.404 & 0.413 & 0.379 & 0.418   \\
         & 336   & \textcolor{red}{\textbf{0.351}} & \textcolor{red}{\textbf{0.396}} & \underline{\textcolor{blue}{0.358}} & 0.403 & 0.374 & 0.408 & 0.365 & \underline{\textcolor{blue}{0.400}} & 0.384 & 0.420 & 0.388 & 0.417 & 0.366 & 0.413 & 0.740 & 0.628 & 0.490 & 0.487 & 0.377 & 0.416 & 0.389 & 0.435 & \underline{\textcolor{blue}{0.358}} & 0.413   \\
         & 720   & \textcolor{red}{\textbf{0.380}} & \textcolor{red}{\textbf{0.423}} & \underline{\textcolor{blue}{0.388}} & \underline{\textcolor{blue}{0.431}} & 0.413 & 0.439 & 0.389 & \textcolor{red}{\textbf{0.423}} & 0.422 & 0.451 & 0.424 & 0.444 & 0.401 & 0.436 & 1.386 & 0.882 & 0.704 & 0.597 & 0.406 & 0.441 & 0.434 & 0.448 & 0.422 & 0.457   \\
         \cmidrule(lr){2-26}
       \multirow{-5}*{\rotatebox{90}{ETTh2}}  & Avg   & \textcolor{red}{\textbf{0.331}} & \textcolor{red}{\textbf{0.381}} & \underline{\textcolor{blue}{0.336}} & \underline{\textcolor{blue}{0.387}} & 0.354 & 0.396 & 0.347 & 0.387 & 0.363 & 0.403 & 0.370 & 0.403 & 0.349 & 0.397 & 0.894 & 0.680 & 0.470 & 0.468 & 0.351 & 0.395 & 0.390 & 0.416 & 0.377 & 0.419   \\
         \midrule
         & 96   & \textcolor{red}{\textbf{0.285}} & \textcolor{red}{\textbf{0.336}} & \underline{\textcolor{blue}{0.289}} & \underline{\textcolor{blue}{0.342}} & 0.296 & 0.349 & 0.332 & 0.361 & 0.293 & 0.347 & 0.300 & 0.353 & 0.293 & 0.345 & 0.314 & 0.367 & 0.300 & 0.345 & \underline{\textcolor{blue}{0.289}} & 0.343 & 0.340 & 0.378 & 0.415 & 0.410   \\
         & 192   & \textcolor{red}{\textbf{0.325}} & \textcolor{red}{\textbf{0.359}} & \underline{\textcolor{blue}{0.327}} & 0.368 & 0.335 & 0.375 & 0.346 & 0.370 & 0.329 & 0.367 & 0.341 & 0.380 & 0.335 & 0.372 & 0.374 & 0.410 & 0.336 & \underline{\textcolor{blue}{0.366}} & 0.329 & 0.368 & 0.392 & 0.404 & 0.494 & 0.451   \\
         & 336   & \textcolor{red}{\textbf{0.359}} & \textcolor{red}{\textbf{0.380}} & \underline{\textcolor{blue}{0.361}} & \underline{\textcolor{blue}{0.382}} & 0.364 & 0.389 & 0.389 & 0.392 & 0.365 & 0.387 & 0.374 & 0.396 & 0.368 & 0.386 & 0.413 & 0.432 & 0.367 & 0.386 & 0.362 & 0.390 & 0.423 & 0.426 & 0.577 & 0.490   \\
         & 720   & \textcolor{red}{\textbf{0.415}} & \textcolor{red}{\textbf{0.412}} & 0.426 & 0.418 & 0.417 & 0.419 & 0.430 & 0.418 & 0.429 & 0.422 & 0.429 & 0.430 & 0.426 & 0.417 & 0.753 & 0.613 & 0.419 & \underline{\textcolor{blue}{0.416}} & \underline{\textcolor{blue}{0.416}} & 0.423 & 0.475 & 0.453 & 0.636 & 0.535   \\
         \cmidrule(lr){2-26}
       \multirow{-5}*{\rotatebox{90}{ETTm1}}  & Avg   & \textcolor{red}{\textbf{0.346}} & \textcolor{red}{\textbf{0.372}} & 0.351 & \underline{\textcolor{blue}{0.378}} & 0.353 & 0.383 & 0.374 & 0.385 & 0.354 & 0.381 & 0.361 & 0.390 & 0.355 & 0.380 & 0.464 & 0.455 & 0.356 & 0.378 & \underline{\textcolor{blue}{0.349}} & 0.381 & 0.407 & 0.415 & 0.530 & 0.472   \\
         \midrule
         & 96   & \textcolor{red}{\textbf{0.161}} & \textcolor{red}{\textbf{0.250}} & \textcolor{red}{\textbf{0.161}} & \underline{\textcolor{blue}{0.251}} & 0.173 & 0.259 & 0.183 & 0.271 & 0.168 & 0.258 & 0.175 & 0.266 & 0.165 & 0.256 & 0.296 & 0.391 & \underline{\textcolor{blue}{0.164}} & 0.255 & 0.165 & 0.255 & 0.189 & 0.265 & 0.210 & 0.294   \\
         & 192   & \underline{\textcolor{blue}{0.217}} & \textcolor{red}{\textbf{0.288}} & \textcolor{red}{\textbf{0.214}} & \textcolor{red}{\textbf{0.288}} & 0.247 & 0.309 & 0.247 & 0.316 & 0.227 & 0.298 & 0.242 & 0.312 & 0.225 & 0.298 & 0.369 & 0.416 & 0.224 & 0.304 & 0.221 & \underline{\textcolor{blue}{0.293}} & 0.254 & 0.310 & 0.338 & 0.373   \\
         & 336   & \underline{\textcolor{blue}{0.270}} & \textcolor{red}{\textbf{0.326}} & \textcolor{red}{\textbf{0.266}} & \textcolor{red}{\textbf{0.326}} & 0.292 & 0.337 & 0.290 & 0.338 & 0.276 & 0.334 & 0.282 & 0.337 & 0.277 & 0.332 & 0.588 & 0.600 & 0.277 & 0.337 & 0.276 & \underline{\textcolor{blue}{0.327}} & 0.313 & 0.345 & 0.432 & 0.416   \\
         & 720   & \underline{\textcolor{blue}{0.359}} & \underline{\textcolor{blue}{0.382}} & \textcolor{red}{\textbf{0.353}} & \underline{\textcolor{blue}{0.382}} & 0.382 & 0.391 & 0.383 & 0.394 & 0.364 & 0.394 & 0.375 & 0.394 & 0.360 & 0.387 & 0.750 & 0.612 & 0.371 & 0.401 & 0.362 & \textcolor{red}{\textbf{0.381}} & 0.413 & 0.402 & 0.554 & 0.476   \\
         \cmidrule(lr){2-26}
       \multirow{-5}*{\rotatebox{90}{ETTm2}}  & Avg   & \underline{\textcolor{blue}{0.252}} & \textcolor{red}{\textbf{0.311}} & \textcolor{red}{\textbf{0.248}} & \underline{\textcolor{blue}{0.312}} & 0.273 & 0.324 & 0.276 & 0.330 & 0.259 & 0.321 & 0.268 & 0.327 & 0.257 & 0.318 & 0.501 & 0.505 & 0.259 & 0.324 & 0.256 & 0.314 & 0.292 & 0.331 & 0.384 & 0.390   \\
         \midrule
         & 96   & \textcolor{red}{\textbf{0.141}} & \textcolor{red}{\textbf{0.193}} & \textcolor{red}{\textbf{0.141}} & \underline{\textcolor{blue}{0.198}} & 0.150 & 0.202 & 0.182 & 0.235 & 0.147 & 0.199 & 0.157 & 0.207 & 0.147 & \underline{\textcolor{blue}{0.198}} & \underline{\textcolor{blue}{0.143}} & 0.210 & 0.170 & 0.230 & 0.150 & 0.200 & 0.168 & 0.214 & 0.188 & 0.242   \\
         & 192   & \textcolor{red}{\textbf{0.183}} & \textcolor{red}{\textbf{0.234}} & \underline{\textcolor{blue}{0.185}} & \underline{\textcolor{blue}{0.236}} & 0.196 & 0.246 & 0.251 & 0.283 & 0.188 & 0.238 & 0.200 & 0.248 & 0.192 & 0.243 & 0.195 & 0.261 & 0.216 & 0.273 & 0.191 & 0.239 & 0.219 & 0.262 & 0.240 & 0.290   \\
         & 336   & \underline{\textcolor{blue}{0.235}} & 0.291 & \textcolor{red}{\textbf{0.232}} & 0.285 & 0.247 & 0.284 & 0.305 & 0.320 & 0.239 & \textcolor{red}{\textbf{0.276}} & 0.252 & 0.287 & 0.247 & 0.284 & 0.254 & 0.319 & 0.258 & 0.307 & 0.242 & \underline{\textcolor{blue}{0.279}} & 0.278 & 0.302 & 0.322 & 0.328   \\
         & 720   & \underline{\textcolor{blue}{0.306}} & 0.345 & \textcolor{red}{\textbf{0.303}} & 0.336 & 0.358 & 0.349 & 0.373 & 0.364 & 0.316 & \textcolor{red}{\textbf{0.328}} & 0.320 & 0.336 & 0.318 & \underline{\textcolor{blue}{0.330}} & 0.335 & 0.385 & 0.323 & 0.362 & 0.312 & \underline{\textcolor{blue}{0.330}} & 0.353 & 0.351 & 0.396 & 0.378   \\
         \cmidrule(lr){2-26}
       \multirow{-5}*{\rotatebox{90}{Weather}}  & Avg   & \underline{\textcolor{blue}{0.216}} & 0.266 & \textcolor{red}{\textbf{0.215}} & 0.264 & 0.238 & 0.270 & 0.278 & 0.300 & 0.222 & \textcolor{red}{\textbf{0.260}} & 0.232 & 0.270 & 0.226 & 0.264 & 0.232 & 0.294 & 0.242 & 0.293 & 0.224 & \underline{\textcolor{blue}{0.262}} & 0.255 & 0.282 & 0.286 & 0.309   \\
         \midrule
         & 96   & \underline{\textcolor{blue}{0.132}} & 0.228 & \textcolor{red}{\textbf{0.129}} & \textcolor{red}{\textbf{0.223}} & 0.137 & 0.228 & 0.220 & 0.291 & \underline{\textcolor{blue}{0.132}} & \underline{\textcolor{blue}{0.227}} & 0.134 & 0.230 & 0.153 & 0.256 & 0.134 & 0.231 & 0.140 & 0.237 & 0.143 & 0.247 & 0.169 & 0.271 & 0.171 & 0.274   \\
         & 192   & 0.148 & \textcolor{red}{\textbf{0.240}} & \underline{\textcolor{blue}{0.147}} & \textcolor{red}{\textbf{0.240}} & 0.152 & \underline{\textcolor{blue}{0.241}} & 0.216 & 0.291 & 0.149 & \underline{\textcolor{blue}{0.241}} & 0.154 & 0.250 & 0.168 & 0.269 & \textcolor{red}{\textbf{0.146}} & 0.243 & 0.154 & 0.251 & 0.158 & 0.260 & 0.180 & 0.280 & 0.180 & 0.283   \\
         & 336   & \textcolor{red}{\textbf{0.162}} & 0.260 & \textcolor{red}{\textbf{0.162}} & 0.261 & 0.166 & \textcolor{red}{\textbf{0.256}} & 0.230 & 0.307 & \underline{\textcolor{blue}{0.165}} & \underline{\textcolor{blue}{0.258}} & 0.169 & 0.265 & 0.189 & 0.291 & \underline{\textcolor{blue}{0.165}} & 0.264 & 0.169 & 0.268 & 0.168 & 0.267 & 0.204 & 0.304 & 0.204 & 0.305   \\
         & 720   & \underline{\textcolor{blue}{0.197}} & 0.296 & 0.199 & 0.299 & 0.204 & \underline{\textcolor{blue}{0.291}} & 0.274 & 0.340 & 0.203 & 0.292 & \textcolor{red}{\textbf{0.194}} & \textcolor{red}{\textbf{0.288}} & 0.228 & 0.320 & 0.237 & 0.314 & 0.204 & 0.301 & 0.214 & 0.307 & 0.205 & 0.304 & 0.221 & 0.319   \\
         \cmidrule(lr){2-26}
       \multirow{-5}*{\rotatebox{90}{Electricity}}  & Avg   & \underline{\textcolor{blue}{0.160}} & 0.256 & \textcolor{red}{\textbf{0.159}} & 0.256 & 0.165 & \textcolor{red}{\textbf{0.254}} & 0.235 & 0.307 & 0.162 & \underline{\textcolor{blue}{0.255}} & 0.163 & 0.258 & 0.184 & 0.284 & 0.171 & 0.263 & 0.167 & 0.264 & 0.171 & 0.270 & 0.189 & 0.290 & 0.194 & 0.295   \\
         \midrule
         & 96   & 0.396 & 0.276 & 0.372 & 0.266 & 0.394 & \underline{\textcolor{blue}{0.260}} & 0.707 & 0.433 & 0.396 & 0.278 & \textcolor{red}{\textbf{0.363}} & 0.265 & \underline{\textcolor{blue}{0.369}} & \textcolor{red}{\textbf{0.257}} & 0.526 & 0.288 & 0.395 & 0.275 & 0.370 & 0.262 & 0.595 & 0.312 & 0.603 & 0.330   \\
         & 192   & 0.408 & 0.279 & 0.389 & 0.274 & \textcolor{red}{\textbf{0.380}} & \textcolor{red}{\textbf{0.250}} & 0.648 & 0.401 & 0.413 & 0.285 & \underline{\textcolor{blue}{0.384}} & 0.273 & 0.400 & 0.272 & 0.503 & \underline{\textcolor{blue}{0.263}} & 0.407 & 0.280 & 0.386 & 0.269 & 0.613 & 0.322 & 0.611 & 0.338   \\
         & 336   & 0.417 & 0.283 & \underline{\textcolor{blue}{0.398}} & 0.279 & 0.429 & 0.278 & 0.668 & 0.407 & 0.421 & 0.291 & \textcolor{red}{\textbf{0.396}} & 0.277 & 0.407 & \textcolor{red}{\textbf{0.272}} & 0.505 & 0.276 & 0.417 & 0.286 & \textcolor{red}{\textbf{0.396}} & \underline{\textcolor{blue}{0.275}} & 0.626 & 0.332 & 0.628 & 0.342   \\
         & 720   & 0.453 & 0.302 & \underline{\textcolor{blue}{0.437}} & 0.301 & 0.440 & \textcolor{red}{\textbf{0.284}} & 0.727 & 0.431 & 0.456 & 0.307 & 0.445 & 0.308 & 0.461 & 0.316 & 0.552 & 0.301 & 0.454 & 0.308 & \textcolor{red}{\textbf{0.435}} & \underline{\textcolor{blue}{0.295}} & 0.635 & 0.340 & 0.646 & 0.350   \\
         \cmidrule(lr){2-26}
       \multirow{-5}*{\rotatebox{90}{Traffic}}  & Avg   & 0.419 & 0.285 & 0.399 & 0.280 & 0.411 & \textcolor{red}{\textbf{0.268}} & 0.688 & 0.418 & 0.421 & 0.290 & \underline{\textcolor{blue}{0.397}} & 0.281 & 0.409 & 0.279 & 0.521 & 0.282 & 0.418 & 0.287 & \textcolor{red}{\textbf{0.397}} & \underline{\textcolor{blue}{0.275}} & 0.617 & 0.327 & 0.622 & 0.340   \\
         \midrule
         \multicolumn{2}{c|}{$1^{st}$ Count} & \textcolor{red}{\textbf{17}} & \textcolor{red}{\textbf{18}} & \underline{\textcolor{blue}{12}} & 4 & 2 & \underline{\textcolor{blue}{5}} & 1 & \underline{\textcolor{blue}{5}} & 0 & 3 & 3 & 1 & 0 & 2 & 1 & 0 & 0 & 0 & 3 & 1 & 0 & 0 & 0 & 0 \\
         \bottomrule[2pt]
    \end{tabular}
    }}
    \caption{Multivariate forecasting comparison across seven datasets with
forecast horizons $\tau\in\{96,192,336,720\}$, under the TFB
protocol~\cite{qiu2024tfb}. Baseline results are taken from TFB, each
reported at its best look-back length over $\{96,336,512\}$; AdaRDiff is
evaluated identically. Best result in \textbf{\textcolor{red}{red}}, second
best \underline{\textcolor{blue}{blue}}; our method in
\textcolor{OursColor}{dark red}.}
    \label{tab:multivariate_main}
\end{table*}

\end{document}